\documentclass{article}
\usepackage{microtype}
\usepackage{graphicx}
\usepackage{subfigure}
\usepackage{booktabs} 
\usepackage{amsthm}
\usepackage{amsmath}
\usepackage{amssymb}
\usepackage{bm}
\usepackage{microtype}
\usepackage{graphicx}
\usepackage{subfigure}
\usepackage{booktabs}
\usepackage{hyperref}
\usepackage[utf8]{inputenc}
\usepackage{layouts}

\newenvironment{proof_sketch}{%
  \proof}{\endproof}
\newtheorem{proposition}{Proposition}
\usepackage[accepted]{icml2019} 

\usepackage{hyperref}

\newcommand{\code}[1]{\texttt{#1}}
\newtheorem{definition}{Definition}
\newtheorem{theorem}{Theorem}
\newtheorem{lemma}{Lemma}
\newtheorem{assumption}{Assumption}
\usepackage{mathtools}
\mathtoolsset{showonlyrefs,showmanualtags}

\icmlsetsymbol{equal}{*}

\icmltitlerunning{More Efficient Off-Policy Evaluation through Regularized Targeted Learning}

\begin{document}

\twocolumn[
\icmltitle{More Efficient Off-Policy Evaluation through Regularized Targeted Learning}

\icmlsetsymbol{equal}{*}

\begin{icmlauthorlist}
\icmlauthor{Aurélien F. Bibaut}{equal,ucberkeley}
\icmlauthor{Ivana Malenica}{equal,ucberkeley}
\icmlauthor{Nikos Vlassis}{netflix}
\icmlauthor{Mark J. van der Laan}{ucberkeley}
\end{icmlauthorlist}

\icmlaffiliation{ucberkeley}{University of California, Berkeley, CA}
\icmlaffiliation{netflix}{Netflix, Los Gatos, CA}

\icmlcorrespondingauthor{Aurélien F. Bibaut}{aurelien.bibaut@berkeley.edu}

\icmlkeywords{off-policy evaluation, TMLE, reinforcement learning}

\vskip 0.3in
]



\printAffiliationsAndNotice{\icmlEqualContribution} 

\begin{abstract}
We study the problem of off-policy evaluation (OPE) in Reinforcement Learning (RL), where the aim is to estimate the performance of a new policy given historical data that may have been generated by a different policy, or policies. In particular, we introduce a novel doubly-robust estimator for the OPE problem in RL, based on the Targeted Maximum Likelihood Estimation principle from the statistical causal inference literature. We also introduce several variance reduction techniques that lead to impressive performance gains in off-policy evaluation. We show empirically that our estimator uniformly wins over existing off-policy evaluation methods across multiple RL environments and various levels of model misspecification. Finally, we further the existing theoretical analysis of estimators for the RL off-policy estimation problem by showing their $O_P(1/\sqrt{n})$ rate of convergence and characterizing their asymptotic distribution.

\textbf{Note: We are uploading the full paper with the appendix as of 12/12/2019, as we noticed that, unlike the main text, the appendix has not been made available on PMLR's website. The version of the appendix in this document is the same that we have been sending by email since June 2019 to readers who solicited it.}
\end{abstract}

\section{Introduction}
\label{intro}

\textit{Off-policy evaluation} (OPE) is an increasingly important problem in reinforcement learning. Works on OPE address the pressing issue of evaluating the performance of a novel policy in a setting where actual enforcement might be too costly, infeasible, or even hazardous. This situation arises in many fields, including medicine, finance, advertising, and education, to name a few \cite{murphy2001,petersen2014ltmle,theocharous2015,hoiles2016}. The OPE problem can be treated as a counterfactual quantity estimation problem, as we inquire about the mean reward we would have accrued, had we, contrary to fact, implemented the policy $\pi_e$ at the time of data-collection. Estimating and inferring such counterfactual quantities is a well studied problem in statistical causal inference, and has led to many methodological developments. One of the things we aim to do in this work is to further earlier efforts \cite{Dudik2011} in bridging the gap between the reinforcement learning and causal inference fields.

There are roughly two predominant classes of approaches to off-policy value evaluation in RL \cite{Jiang2015}. The first is the \textit{direct method} (DM), analogous to the \textit{G-computation} procedure in causal inference \cite{Gcomp1999,Gcomp2000}. The direct method first fits a model of the system's dynamics and then uses the learned fit in order to estimate the mean reward of the target policy (evaluation policy). The estimators produced by this approach usually exhibit low variance, but suffer from high bias when the model fit is misspecified or the sample size is small relative to the complexity of the function class of the model \cite{mannor2007}. The second major avenue for off-policy value evaluation is \textit{importance sampling} methods, also termed \textit{inverse propensity score} methods in statistical causal inference \cite{rosenbaum1983}. Importance sampling (IS) attempts to correct the mismatch between the distributions produced by the behavior and target policies \cite{Precup2000,PrecupThesis}. IS estimators are unbiased under mild conditions, but their variance tends to be large when the evaluation and behavior policies differ significantly \cite{Fara2018}, and grows exponentially with the horizon, rendering them \cite{Fara2018} impractical for many RL settings. A third class of estimators, \textit{Doubly Robust} (DR) estimators, obtained by combining a DM estimator and an IS estimator, are becoming standard in OPE \cite{Fara2018,Jiang2015,thomas2016}. These originate from the statistics literature \cite{robins1994,robins1995,bang2005,tmle2006,book2011,book2018}, and were introduced in the RL literature by \citet{Dudik2011}. Combining a DM and an IS estimator under the form of a DR estimator leads to lower bias than DM alone, and lower variance than IS alone.

Our contribution to OPE in RL is multifold. First we adapt a doubly robust estimator from statistical causal inference, the Longitudinal Targeted Maximum Likelihood Estimator (LTMLE) to the OPE in RL setting. We show that our adapted estimator converges at rate $O_P(1/\sqrt{n})$ to the true policy value. Deriving the LTMLE requires us to identify a mathematical object known in semiparametric statistics as the \textit{efficient influence function} (EIF) of the estimand (policy value). To the best of our knowledge, this article is the first one to explicitly derive the EIF of the policy value for the OPE problem in RL. Knowledge of the EIF allows us to prove that both our estimator (the LTMLE) and recently proposed DR estimators \citep{Jiang2015, thomas2016} are optimal in the sense that they achieve the generalized Cramer-Rao lower bound.

Second, we introduce an idea from statistics to make better use of the data than prior OPE works \citep{Jiang2015, thomas2016}. We noticed that most OPE papers, at least in theory, use sample splitting: the $Q$-function is fitted on a split of the data, while the DR estimator is obtained by evaluating the fitted $Q$-function on another split. We propose a cross-validation-based technique that allows to essentially average the $Q$-function over the entire sample, leading to a constant-factor gain in risk.

Finally, and most importantly for practice, we propose several regularization techniques for the LTMLE estimators, out of which some, but not all, apply to other DR estimators. Using the MAGIC ensemble method from \citet{thomas2016}, we construct an estimator that combines various regularized LTMLEs. We call our estimator RLTMLE (TMLE for RL). Our experiments demonstrate that RLTMLE outperforms all considered competing off-policy methods, uniformly across multiple RL environments and levels of model misspecification.

\section{Statistical Formulation of the Problem}\label{section-statistical_formulation}
\subsection{Markov Decision Process}
Consider a Markov Decision Process (MDP) defined as a tuple $(\mathcal{S},\mathcal{A}, \mathcal{R}, P_1, P, \gamma)$, where $\mathcal{S}$ and $\mathcal{A}$ are the state and action spaces, and $\gamma \in (0,1]$ is a discount factor. A trajectory $H$ is a succession of states $S_t$, actions $A_t$ and rewards $R_t$, observed from $t=1$ to the horizon $t=T$: $H=(S_1,A_1,R_1,...,S_T, A_T, R_T)$. For all $(s,a,r,s') \in \mathcal{S} \times \mathcal{A}\times \mathcal{R} \times \mathcal{S}$,  $P(s', r|s, a)$ is the probability of collecting reward $r$ and transitioning to state $s'$, conditional on starting in state $s$ and taking action $a$, and $P_1(s)$ is the probability that the initial is $s$. A policy $\pi$ is a sequence of conditional distributions $(\pi_1, \pi_2,... )$ that stochastically map a state to an action: for all $t$, $A_t|S_t \sim \pi_t$. 

Suppose we are given $n$ i.i.d. $T$-step trajectories of the MDP, $D=(H_1,...,H_n)$, collected under the behavior policy $\pi_b=(\pi_{b,1},....,\pi_{b,T})$. We assume all trajectories have the same initial state $s_1$, allowing for the data-generating mechanism to be fully characterized by $(P, \pi_b)$.

\subsection{Estimation Target}
The goal of OPE is to estimate the average cumulative discounted reward we would have obtained by carrying out the target policy $\pi_e$ instead of policy $\pi_b$. That is, we want to estimate the following counterfactual quantity:
\begin{equation}\label{counterfactual_target}
    V_1^{\pi_e}(s_1) := E_{P, \pi_e}\left[\sum_{t=1}^T \gamma^{t} R_t |S_1=s_1\right]. 
\end{equation}
Consider the following common assumption from the causal inference literature.
\begin{assumption}[Absolute continuity]\label{assumption_absolute_continuity}
For all $s, a \in \mathcal{S} \times \mathcal{A}$, if $\pi_b(a|s) = 0$, then $\pi_e(a|s) = 0$ too.
\end{assumption}
Under assumption \ref{assumption_absolute_continuity} and the Markov assumption of the MDP model,  $V_1^{\pi_e}(s_1)$  can be written as an expectation under the data-generating mechanism $(P, \pi_b)$:
\begin{align}\label{identified_target}
    V_1^{\pi_e}(s_1) = E_{P, \pi_b}\left[\prod_{t=1}^T \frac{\pi_{e,t}(A_t|S_t)}{\pi_{b,t}(A_t|S_t)} \sum_{t=1}^T \gamma^{t} R_t \bigg| S_1 = s_1\right]. 
\end{align}
For $t=1,...,T$, define $\bar{R}_{t:T} := \sum_{\tau=t}^T \gamma^{\tau-t} R_\tau$ as the total reward from step $t$ to step $T$. For all $1\leq t_1 \leq t_2 \leq T$, define $\rho_{t_1 : t_2} := \prod_{\tau=t_1}^{t_2} \pi_{e,\tau}(A_\tau|S_\tau) / \pi_{b, \tau}(A_\tau|S_\tau)$. For all $t = 1,...,T$, we will use the shortcut notation $\rho_t := \rho_{1:t}$. We use the convention that $\rho_0 = 0$. Denote $\bar{R}^{(i)}_{t:T}$, $\rho_t^{(i)}$, $\rho_{t_1:t_2}^{(i)}$ the corresponding quantities for a sample trajectory $H_i$. Consistently with \eqref{counterfactual_target} and \eqref{identified_target}, we define, for any $t=1,...,T$, and $s \in \mathcal{S}$, the value function (or reward-to-go) from time point $t$ and state $s$, as
\begin{align}\label{whyim}
   V_t^{\pi_e}(s) :&= E_{P, \pi_e} [ \bar{R}_{t:T} |S_t = s] \\
   &= E_{P, \pi_b}\left[ \rho_{t:T} \bar{R}_{t:T}|S_t=s \right].
\end{align} 
For every $t=1,...,T$, $s\in \mathcal{S}$, $a\in \mathcal{A}$, we further define the action-value function from time step $t$ as
\begin{align}
    Q_t^{\pi_e}(s,a) &:= E_{P, \pi_e}\left[ \bar{R}_{t:T}|S_t = s, A_t= a \right] \nonumber \\
    &= E_{P, \pi_b}\left[\rho_{t:T} \bar{R}_{t:T} | S_t = s, A_t=a \right].
\end{align}

\section{An existing state-of-the art approach}\label{section_existing_work}

Our method can be seen as building upon and improving on \citet{thomas2016}. We believe it helps understanding our contribution to first briefly describe their estimators. For a detailed review of OPE methods, we refer the interested reader to the vast and excellent literature on the topic \cite{Precup2000,ThomasThesis,Jiang2015,Fara2018}.

\subsection{Weighted Doubly Robust Estimator}

\citet{Jiang2015} were the first authors to propose a doubly robust estimator for off-policy evaluation in the MDP setting. \citet{thomas2016} propose a stabilized version of the DR estimator of \citet{Jiang2015}, termed Weighted Doubly Robust (WDR) estimator, which they obtain by replacing the importance sampling weights by stabilized importance sampling weights. The stabilized importance sampling weight for observation $i$ at time step $t$ is defined as $w_t^{(i)} = \rho_t^{(i)} / \sum_{i=1}^n \rho_t^{(i)}$. The WDR estimator is thus defined as
\begin{align}
    &WDR := \sum_{i=1}^n \bigg\{ \frac{1}{n} V_1^{\pi_e}(S_1^{(i)}) \\
    &+ \sum_{t=1}^{T} \gamma^t w_t^{(i)} \left[R_t^{(i)} - Q^{\pi_e}_t(S_t^{(i)},A_t^{(i)}) 
    + \gamma V_{t+1}^{\pi_e}(S_{t+1}^{(i)})\right] \bigg\}. \label{eq_WDR_definition}
\end{align}

\subsection{MAGIC}

While WDR has low bias and converges at rate $O_P(1/\sqrt{n})$ to the truth, its reliance on importance weights can make it highly variable. As a result, in some settings, especially if model misspecification is not too strong, DM estimators can beat WDR \cite{thomas2016}. This motivates the construction of an estimator that interpolates between DM and WDR, so as to benefit from the best of both worlds. \citet{thomas2016} propose the \textit{partial importance sampling} estimators, which correspond to essentially cutting off the sum in \eqref{eq_WDR_definition} the terms with index $t \geq j$ for some $0 \leq j \leq T$. Formally, they define their partial importance sampling estimator as the average $g_j := \sum_{i=1}^n g_j^{(i)}$ of the so-called \textit{off-policy $j$-step return}, that they define, for each trajectory $i$, as
\begin{align}
    g_i^{(j)} &:= \sum_{t=1}^{j} \underbrace{\gamma^t w_t^i R_t^{(i)}}_\text{a} + \underbrace{\gamma^{j+1} w_j^i V_{j+1}^{\pi_e}(S_{j+1}^i)}_\text{b} \nonumber \\
    &- \sum_{t=1}^j
    \underbrace{\gamma^t[w_t^i Q_t^{\pi_e}(S_t^{(i)},A_t^{(i)}) - 
    w_{t-1}^i V_t^{\pi_e}(S_t^{(i)})]}_\text{c},
\end{align}
Note that $g_0$ is equal to the DM estimator. Note that the last component, (c), represents the combined control variate for the importance sampling (a) and model based term (b). Hence, as $j$ increases, we expect bias to decrease, at the expense of an increase in variance.

\citet{thomas2016}'s final estimator is a convex combination of the partial importance sampling estimators $g_j$. Ideally, we would like this convex combination to minimize mean squared error (MSE), that is we would like to use as estimator $(\textbf{x}^*)^\top \bm{g}$, with $\bm{g} =  (g_0,...,g_T)$, where 
\begin{align}
    \bm{x}^* &= \arg \min_{\substack{0 \leq \bm{x} \leq 1 \atop \sum_{j=0}^T x_j = 1}} \text{MSE}(\bm{x}^\top \bm{g}, V^{\pi_e}_1)\\
    &= \arg \min_{\substack{0 \leq \bm{x} \leq 1 \atop \sum_{j=0}^T x_j = 1}} \bigg\{ \text{Bias}^2(\bm{x}^\top \bm{g}, V^{\pi_e}_1) \\
    &\qquad \qquad + \text{Var}(\bm{x}^\top \bm{g}) \bigg\}. \label{magic_x_star}
\end{align}
As we do not have access to the true variance and bias, \citet{thomas2016} propose to use as estimator $\hat{\bm{x}}^\top \bm{g}$, where $\hat{\bm{x}}$ is a minimizer, over the convex weights simplex, of an estimate of the MSE. The covariance matrix of $\bm{g}$, which we will denote $\bm{\Omega}_n$, can be estimated as the empirical covariance matrix $\bm{\hat{\Omega}}_n$ of the $\bm{g}^{(i)}$'s. Bias estimation is a more involved. For each $j=1,...,T$, \citet{thomas2016} estimate the bias of the partial importance sampling estimator $g_j$ by its distance to a $\delta$-confidence interval for $g_T$ obtained by bootstrapping it, for some $\delta \in (0,1)$. They named the resulting ensemble estimator MAGIC, standing for \textit{model and guided importance sampling combining}. For further details, we refer the reader to the very clear presentation of their algorithm by \citet{thomas2016}.

\section{Longitudinal TMLE for MDPs}

\subsection{High level description}
Our proposed estimator extends the longitudinal Targeted Maximum Likelihood Estimation (TMLE) methodology, initially developed in the statistics causal inference literature, to the MDP setting \cite{tmle2006,vdl2011,book2011,book2018}. In order to build intuition on our estimator, we start with a high-level description. Targeted Maximum Likelihood Estimation is a general framework that allows to construct efficient nonparametric estimators of low-dimensional characteristics of the data-generating distribution, given machine learning based estimators of high-dimensional characteristics. Let us illustrate on an example what these low-dimensional and high-dimensional characteristics can be. Suppose we want to estimate an average treatment effect (ATE), and that we have pre-treatment covariates $X$, a treatment $T$ and an outcome $Y$, with $(X,T,Y) \sim P$. In this situation, the low-dimensional characteristic is the ATE $E_P[E_P[Y|T=1,X] - E_P[Y|T=0,X]]$, while the high-dimensional characteristics of $P$ are the outcome regression function $x,a \mapsto E_P[Y|A=a, X=x]$ and the propensity score function $x \mapsto E_P[T|X=x]$.

\subsection{Simplified sample-splitting based algorithm} 
In the following sections we present a simplified version of the algorithm that constructs our Longitudinal Targeted Maximum Likelihood Estimator. The full-blown version of the algorithm is presented in the appendix, with the corresponding theoretical justifications.

Suppose we are provided with $n$ i.i.d. trajectories, $D=(H_1,...,H_n)$. Make two splits of the sample: for some $0 < p <1$, let $D^{(0)}=(H_1,...,H_{(1-p)n})$ and $D^{(1)}=(H_{(1-p)n + 1},...,H_n)$. Use $D^{(0)}$ to fit estimators $\hat{Q}^{\pi_e}_1,\cdots,\hat{Q}^{\pi_e}_T$ of the action value functions $Q_1^{\pi_e},\cdots,Q_T^{\pi_e}$ We will call $\hat{Q}^{\pi_e}_1,\cdots,\hat{Q}^{\pi_e}_T$ the \textit{initial estimators}. Such estimators can be obtained for instance by fitting a model of the dynamics of the MDP, or by SARSA, among other methods \cite{Sutton1998}. Estimators fitted in such a way tend to exhibit low variance but often suffer from misspecification bias. As mentioned in section \ref{section_existing_work}, doubly-robust estimators take such initial estimators as input, and evaluate on $D^{(1)}$ and then average a certain function of them to produce an unbiased estimator of $V^{\pi_e}_1(s_1)$. These doubly-robust estimators rely on the addition of terms weighted by the importance sampling (IS) ratios $\rho^{(i)}_{i:t}$, $i=1,\cdots,n$, $t=1,\cdots,n$. The TMLE methodology takes another route: for each $t$, it defines, on top of the initial estimator fit, a parametric model, which we will call a \textit{second-stage parametric model} $\hat{Q}_t^{\pi_e}$, and achieves bias reduction by fitting this parametric model by maximum likelihood, on the sample split $D^{(1)}$.

\subsection{Formal presentation of the simplified algorithm}\label{section_formal_presentation_simplified_algo}

To formally describe our algorithm, it suffices to define the second-stage parametric models and describe the loss used for the fit. For all $x \in \mathbb{R}$, we define $\sigma(x) = 1 / (1 + e^{-x})$ as the logistic function, and we denote $\sigma^{-1}$ its inverse. Observe that bounding the range of rewards where $\forall t, R_t \in [r_{min}, r_{max}]$, implies that $\forall t$ and $\forall (s,a) \in \mathcal{S} \times \mathcal{A}$, $Q_t(s,a) \in [-\Delta_t, \Delta_t]$ with $\Delta_t := \sum_{\tau=t}^T \gamma^{\tau-t} \max(r_{max},|r_{min}|)$. We further denote $\tilde{Q}^{\pi_e}_t(s, a):= (\hat{Q}^{\pi_e}_t + \Delta_t) / (2 \Delta_t)$ as the normalized initial estimator. In addition, $\forall \delta \in (0, 1/2)$ and $\forall (s,a)$, we define the following thresholded version of $\tilde{Q}_t^{\pi_e}$:
\begin{align}
    \tilde{Q}_t^{\pi_e, \delta}(s, a) := \begin{cases}
    1-\delta & \text{ if }  \tilde{Q}_t^{\pi_e}(s, a) > 1 - \delta,\\
    \tilde{Q}_t^{\pi_e}(s, a) &\text{ if } \tilde{Q}_t^{\pi_e}(s, a) \in [\delta, 1- \delta], \\
    \delta & \text{ if } \tilde{Q}_t^{\pi_e}(s, a) < \delta.
    \end{cases}
\end{align}
For all $\epsilon \in \mathbb{R}$, we can now define the normalized version of our second-stage parametric model as:
\begin{equation}
    \tilde{Q}^{\pi_e, \delta}_t(\epsilon)(s,a) := \sigma(\sigma^{-1}(\tilde{Q}_t^{\pi_e, \delta}(s, a)) + \epsilon). \label{update_t}
\end{equation}
Finally, we denote $\hat{Q}_t^{\pi_e, \delta}(\epsilon) = 2 \Delta_t (\tilde{Q}^{\pi_e, \delta}_t(\epsilon) - 1/2)$ as the rescaled version of $\tilde{Q}^{\pi_e, \delta}_t(\epsilon).$

The normalization, thresholding and rescaling steps in the definition of the parametric second-stage model ensure that (1) $\tilde{Q}^{\pi_e, \delta}_t(\epsilon) \in [\delta, 1- \delta] \subset (0,1)$ for all $\epsilon$, and that (2) $\hat{Q}_t^{\pi_e, \delta}(\epsilon)$ always stays in the allowed range of rewards $[-\Delta_t, \Delta_t]$. The definition of $\tilde{Q}^{\pi_e, \delta}_t(\epsilon)$ as a logistic transform of $\epsilon$ that lies in $(0,1)$ makes the fitting of $\epsilon$ possible through maximum likelihood for a logistic likelihood. For $t=T$, since $Q_T^{\pi_e}(s,a) = E_{P, \pi_b}[\rho_{1:T} R_T|S_T=s, A_T=a]$, it is natural to consider the log likelihood,
\begin{align}
    \mathcal{R}^{\delta}_{n,T}&(\epsilon) = \frac{1}{n} \sum_{i=1}^n \rho^{(i)}_{1:T} \bigg(\tilde{U}_T^{(i)} \log(\tilde{Q}^{\pi_e, \delta}_T(\epsilon)(S_T^{(i)}, A_T^{(i)})) \nonumber \\
    &+(1-\tilde{U}_T^{(i)}) \log(1-\tilde{Q}^{\pi_e, \delta}_T(\epsilon)(S_T^{(i)}, A_T^{(i)})) \bigg), \label{log_likelihood_T}
\end{align}
where $\tilde{U}_T^{(i)} := (R_T^{(i)} + \Delta_T) /  (2 \Delta_T)$ is the normalized reward at time $T$. Normalization of the reward is necessary since we are using logistic regression to optimize $\epsilon$, and to keep the definition of $\tilde{U}_T^{(i)}$ and $\tilde{Q}^{\pi_e, \delta}_T(s,a)$ consistent. The thresholding step that defines $\tilde{Q}^{\delta}_t(s,a)$ prevents the log likelihood from taking on non-finite values. In order to make the bias introduced by thresholding vanish as the sample size grows, we use a vanishing sequence $\delta_n \downarrow 0$ of thresholding values.

Let $\epsilon_{n,T}$ be the minimizer over $\mathbb{R}$ of the log likelihood $\mathcal{R}_{n,t}^{\delta}$ for step $T$. We fit the second-stage models for $t=T-1,...,1$ by backward recursion, a procedure which we describe in more detail in this paragraph.
Start with observing that for all $t=1,...,T$, and for all $(s,a) \in \mathcal{S} \times \mathcal{A}$, $Q^{\pi_e}_t(s,a) = E_{\pi_b}[\rho_{1:t} (R_t + \gamma V_{t+1}^{\pi_e}(S_{t+1}))|S_t=s, A_t=a]$. This motivates defining, as outcome of the rescaled logistic regression model for time step $t$, the normalized reward-to-go:
\begin{equation}
   \tilde{U}^{(i)}_{t,n} := (R^{(i)}_t + \gamma \hat{V}_{t+1}^{\pi_e}(\epsilon_{n, t+1})(S^{(i)}_{t+1}) + \Delta_t) / (2 \Delta_t). 
\end{equation}
Define $\hat{V}_{t}^{\pi_e}(\epsilon)$ as the value function corresponding to the action-value function $\hat{Q}_t^{\pi_e, \delta_n}(\epsilon)$, that is, for all $s \in \mathcal{S}$, set $\hat{V}_{t}^{\pi_e}(\epsilon)(s) = \sum_{a' \in \mathcal{A}} \pi_e(a'|s) \hat{Q}^{\pi_e, \delta_n}(\epsilon)(s, a')$. We define the second-stage model log likelihood for each $t=T-1,...,1$ as
\begin{align}
    \mathcal{R}_{t,n}^{\delta}(\epsilon) &= \frac{1}{n} \sum_{i=1}^n \rho^{(i)}_{1:t} \bigg(\tilde{U}_t^{(i)} \log(\tilde{Q}_t^{\pi_e, \delta}(\epsilon)(S_t^{(i)}, A_t^{(i)})) \nonumber \\
    &+(1-\tilde{U}_t^{(i)}) \log(1-\tilde{Q}_t^{\pi_e, \delta}(\epsilon)(S_t^{(i)}, A_t^{(i)})) \bigg). \label{log_likelihood_t}
\end{align}  
The fact that the outcome in the second-stage logistic model at time step $t$ depends on the second-stage model fit at time step $t+1$ is why we have to proceed backwards in time. This is why we say this procedure is a \textit{backward recursion}.

Finally, once all of the $T$ second-stage models have been fitted, we define the  LTMLE estimator of $V_1^{\pi_e}(s_1)$ as follows:
\begin{align}
     \hat{V}_1^{\pi_e, LTMLE}(s_1) := \hat{V}_1^{\pi_e}(\epsilon_{n, 1})(s_1).
\end{align}
This idea of backward recursion we just exposed was initially introduced in \cite{bang2005}. They called it \textit{sequential regression}.

We present the pseudo-code of the procedure as Algorithm \ref{alg:ltmle}.

\begin{algorithm}[tb]
   \caption{Longitudinal TMLE for MDPs}
   \label{alg:ltmle}
\begin{algorithmic}
   \STATE {\bfseries Input:} Logged data split $D^{(1)}$, target policy $\pi_e$, initial estimators $\hat{Q}^{\pi_e}_1,...,\hat{Q}^{\pi_e}_T$, discount factor $\gamma$.
   \STATE Set $\Delta_{T}=0$ and $\hat{V}_{T+1}^{\pi_e}=\bm{0}$.
   \FOR{$t=T$ {\bfseries to} $1$}
   \STATE Set $\Delta_t= \max_{t,i} |R_t| + \gamma\Delta_{t}$.
   \STATE Set $\tilde{U}_t = (R_t + \gamma 
                                 \hat{V}_{t+1}^{\pi_e} + \Delta_t) / 2\Delta_t$.
    \STATE Set $\tilde{Q}_t^{\pi_e, \delta_n} = \text{threshold}(\delta_n, (\hat{Q}^{\pi_e}_t + \Delta_t) / 2\Delta_t)$.
    \STATE Compute $\epsilon_{n,t} = \arg \min_{\epsilon} \mathcal{R}_{n,t}^{\delta_n}(\epsilon)$.
    \STATE Set $\hat{Q}_t^{\pi_e, \delta_n} = 2 \Delta_t (\tilde{Q}_t^{\pi_e, \delta_n} - 0.5)$.
    \STATE Set, for all $s \in \mathcal{S}$, $$\hat{V}_t^{\pi_e}(s) = \sum_{a' \in \mathcal{A}} \pi_e(a'|s) \hat{Q}_t^{\pi_e, \delta_n}(s, a').$$
   \ENDFOR
   \STATE \textbf{return} $\hat{V}_{1}^{\pi_e}(\epsilon_{n,1})(s_1)$.
\end{algorithmic}
\end{algorithm}

\subsection{Guarantees and benefits}

It might at first appear surprising that fitting the second-stage models, which amounts to simply fitting the intercept of a logistic regression model, suffices to fully remove the bias. We nevertheless prove that it does so in theorem \ref{theorem_simplified_algorithm} under mild assumptions. Theorem \ref{theorem_simplified_algorithm} requires assumption \ref{assumption_absolute_continuity} stated in section \ref{section-statistical_formulation} and assumptions 2-4 stated below.

\begin{assumption}\label{assumption_bounded_rewards}
For all $t=1,....,T$, $r_t \in [r_{min},r_{max}]$ almost surely.
\end{assumption}

\begin{assumption}\label{assumption_convergence_initial_estimator}
For all $t=1,...,T$, the initial estimator $\hat{Q}^{\pi_e}_{t,n}$ converges in probability to some limit $Q_{t, \infty}: \mathcal{S} \times \mathcal{A} \rightarrow \mathbb{R}$, that is $\| \hat{Q}^{\pi_e}_{t,n} - Q_{t, \infty} \|_{P,2} = o_P(1)$.
\end{assumption}

\begin{assumption}\label{assumption_bounds_limit_initial_estimator}
For all $t=1,...,T$, let $Q_{t, \infty}$ be the limit as defined in Assumption \ref{assumption_convergence_initial_estimator}. Assume there exists a (small) positive constant $\eta \in (0,1/2)$ such that $\forall t$ and $\forall (s,a) \in \mathcal{S} \times \mathcal{A}$, $Q_{t, \infty}(s,a)\in [\eta, 1-\eta]$.
\end{assumption}

\begin{assumption}\label{assumption_bound_IS_ratios} Suppose there exists a finite positive constant $M$ such that $\forall t$,
$\rho_{1:t} \leq M$ almost surely.
\end{assumption}

We can now state our main theoretical result, for the algorithm presented in section \ref{section_formal_presentation_simplified_algo}. 
\begin{theorem}\label{theorem_simplified_algorithm}
Suppose assumptions \ref{assumption_bounded_rewards}, \ref{assumption_convergence_initial_estimator}, \ref{assumption_bounds_limit_initial_estimator}, and \ref{assumption_bound_IS_ratios} hold. Then the LTMLE estimator has bias $o(1 / \sqrt{n})$, that is $$E_{P, \pi_b}[\hat{V}_1^{\pi_e, LTMLE}(s_1)] - V^{\pi_e}_1(s_1) = o(1 /\sqrt{n}).$$
In addition, the LTMLE estimator converges in probability at rate $\sqrt{n}$, that is
\begin{align}
\hat{V}_1^{\pi_e, LTMLE}(s_1) - V^{\pi_e}_1(s_1) = O_P(1 / \sqrt{n}).
\end{align}
\end{theorem}

With a little extra work, we can also characterize the asymptotic distribution and the asymptotic variance of the LTMLE estimator. In particular, we show in the appendix that, provided that $\hat{Q}^{\pi_e}$ is consistent, our estimator attains the generalized Cramer-Rao bound and is therefore \textit{locally efficient}. We also argue that it is asymptotically equivalent with the doubly robust estimator \cite{thomas2016,Jiang2015}.

\section{RLTMLE}

In this section, we (1) present regularizations that can be applied to the LTMLE estimator, and (2) describe our ``final estimator'', which we call RLTMLE (standing for \textit{LTMLE for RL}), and which consists of a convex combination of regularized LTMLE estimators. The weights in the RLTMLE convex combination are obtained following a variant of the ensembling procedure of the MAGIC estimator.

\subsection{Regularization and base estimators} \label{regularization}

We present three regularization techniques that allow to stabilize the variance of the LTMLE estimator. The first two have a clear WDR analogue, while the third one only applies to LTMLE.

\begin{enumerate}
    \item \textbf{Weights softening.} For $\alpha \in [0,1]$, $x \in \mathbb{R}^d$, define $\linebreak\code{soften}(x, \alpha) := (x_k^\alpha / \sum_{l=1}^d x_l^\alpha : k = 1,...,d)$. The LTMLE algorithm corresponding to softening level $\alpha$ is obtained by replacing, in the second-stage log likelihoods \eqref{log_likelihood_T} and \eqref{log_likelihood_t}, the IS ratios $(\rho^{(i)}_{1:t}:i=1,...,n)$ by $\code{soften}( (\rho^{(i)}_{1:t}:i=1,...,n), \alpha)$. The same operation can be applied as well to the importance weights of the WDR estimator.
    
    \item \textbf{Partial horizon.} The LTMLE with partial horizon $\tau < T$ is obtained by setting to zero the coefficients $\epsilon_{n,{\tau_1}},..., \epsilon_{n,T}$ before fitting the other second-stage coefficients. This enforces that the importance sampling ratios $\rho_{1:t}$ for $t \geq j$ have no impact on the estimator. The WDR equivalent is to use the $\tau$-step return $g_\tau$.

    \item \textbf{Penalization.} The penalized LTMLE is obtained by adding a penalty $\lambda |\epsilon_{n,t}|$ for some $\lambda \geq 0$ to the the log-likelihoods \eqref{log_likelihood_T} and \eqref{log_likelihood_t} of the second-stage models.
\end{enumerate}

The three regularizations can be applied simultaneously. A regularized LTMLE estimator can therefore be indexed by a triple $(\alpha, \tau, \lambda)$, where $\alpha$, $\tau$ and $\lambda$ denote the level of softening, the partial horizon, and the level of likelihood penalization.

\subsection{Ensemble estimator}

Our final estimator is an ensemble of a pool of regularized LTMLE estimators, which we will denote $g_1,...,g_K$, that correspond to a sequence of triples $(\alpha_1, \tau_1, \lambda_1),..., (\alpha_K, \tau_K, \lambda_K)$ of regularization levels. We set $g_K$ to be the unregularized LTMLE, that is we set $(\alpha_K, \tau_K, \lambda_K) = (1, T, 0)$. We ensemble the regularized LTMLE estimators $g_1,...,g_K$ by taking a convex combination of them that minimizes an estimate of MSE. The ensembling step closely follows that of the MAGIC procedure. We propose two variants of it, which we call RLTMLE 1 and RLTMLE 2, differing in how we estimate the covariance matrix $\bm{\Omega}_n$ (defined in section \ref{section_existing_work}) of base estimators $g_1,...,g_K$.

\paragraph{RLTMLE 1.} In this variant of RLTMLE, covariance estimation relies on the following property of the LTMLE estimator. As we show in the appendix, the difference between a regularized LTMLE estimator with regularization parameters $(\alpha, \tau, \lambda)$, and its asymptotic limit is given by
$n^{-1} \sum_{i=1}^n \text{EIF}(\bm{\hat{Q}}, \alpha, \tau, \lambda) (H_i) + o_P(n^{-1/2})$, where $\text{EIF}$ is the efficient influence function, presented in the appendix, whose expression is given by
\begin{align}
    &\text{EIF}(\bm{\hat{Q}}^{\pi_e}, \alpha, \lambda, \tau)(h)\\
    &= \sum_{t=1}^T \gamma^t \rho_t \times 
     \big( r_t + \gamma \hat{V}_{t+1}^{\pi_e}(\epsilon_{n,t+1})(s_{t+1}) \\  &-\hat{Q}_{t}^{\pi_e}(\epsilon_{n,t})(s_t,a_t)\big) \label{EIF},
\end{align}
where, for all $t$, $\epsilon_{n,t}$ is the maximizer of the regularized version of the log-likelihood \eqref{log_likelihood_t}, that is expression \eqref{log_likelihood_t} where $\rho_t$ is replaced with $\text{soften}(\rho_t, \alpha)$ and penalized by $\lambda |\epsilon|$. Denote $\text{EIF}_k(h) = \text{EIF}(\bm{\hat{Q}}, \alpha_k, \lambda_k, \tau_k)(h)$, the EIF corresponding to estimator $g_k$. We use as estimate of the covariance matrix $\bm{\Omega}_n$ the empirical covariance matrix $\bm{\hat{\Omega}}_n$ of $(\text{EIF}_1(H),..., \text{EIF}_K(H))$.

\paragraph{RLTMLE 2.} In this variant of RLTMLE, an estimate of the covariance matrix $\bm{\Omega}_n$ of the base estimators $\bm{g} =(g_1,...,g_K)$ is obtained by computing bootstrapped values $\bm{g}^{(1)},...,\bm{g}^{(B)}$, of $\bm{g}$, for a large enough number of bootstrap samples $B$, and computing the empirical covariance $\bm{\hat{\Omega}}_n$ matrix of $\bm{g}^{(1)},...,\bm{g}^{(B)}$.

\paragraph{Bias estimation.} We follow closely \citet{thomas2016} for bias estimation. For $k=1,...,K$, denote $b_{n,k}$ the bias of estimator $g_K$, and $\bm{b}_n := (b_{n,1},...,b_{n,K})$. Denote $\text{CI}(\alpha)$ the $\alpha$-percentile bootstrap confidence interval for the LTMLE estimator. In both RLTMLE 1 and RLTMLE 2, for each $k=1,...,K$, estimate the bias $b_{n,k}$ with $\hat{b}_{n,k} := \text{dist}(g_k, \text{CI}(\alpha))$. Denote $\bm{\hat{b}}_n := (\hat{b}_{n,1}, ... , \hat{b}_{n,K})$.

Because of space limitation, we only give a pseudo-code description of RLTMLE 2, which is our most performant algorithm, as we will see in the next section.

\begin{algorithm}[tb]
   \caption{RLTMLE 2}
   \label{alg:rltmle}
\begin{algorithmic}
\STATE {\bfseries Input:} Logged data split $D^{(1)}$, target policy $\pi_e$, initial estimator $\bm{\hat{Q}}^{\pi_e} := (\hat{Q}^{\pi_e}_1,...,\hat{Q}^{\pi_e}_T)$, discount factor $\gamma$, triples of regularization levels $(\alpha_1, \tau_1, \lambda_1),...,(\alpha_K, \tau_K, \lambda_K)$, number of bootstrap samples $B$.

   \FOR{$b=1$ {\bfseries to} $B$}
   \STATE Sample with replacement from $D^{(1)}$  a bootstrap sample $D^{*, (b)}$.
    \FOR{$k=1$ {\bfseries to} $K$} 
    \STATE Compute $g_k^{(b)}$ by running algorithm \ref{alg:ltmle} with inputs $D^{*, (b)}$, $\bm{\hat{Q}}^{\pi_e}$, $\pi_e$, $\gamma$, using regularizations levels $(\alpha_k, \tau_k, \lambda_k)$.
    \ENDFOR
   \ENDFOR
   \FOR{$k = 1$ {\bfseries to} $K$}
   \STATE Compute $g_k$ by running algorithm \ref{alg:ltmle} with inputs $D^{(1)}$, $\bm{\hat{Q}}^{\pi_e}$, $\pi_e$, $\gamma$, using regularizations levels $(\alpha_k, \tau_k, \lambda_k)$.
    \FOR{$l = 1$ {\bfseries to} $K$}
       \STATE $\bm{\hat{\Omega}}_{k,l} \gets n^{-1} \sum_{b=1}^B g_k^{(b)} g_l^{(b)} - \left(n^{-1} \sum_{b=1}^B g_k^{(b)}\right) \left( n^{-1} \sum_{b=1}^B g_l^{(b)} \right)$.
    \ENDFOR
    \STATE $\text{CI}(\alpha) \gets \big[\text{percentile}(\{g^{(b)}_k : b\}, \alpha), \text{percentile}(\{g^{(b)}_k : b\}, 1-\alpha)\big]$.
    \STATE $\hat{b}_{n,k} \gets \text{distance}(g_k, \text{CI}(\alpha))$.
   \ENDFOR
  \STATE $$\bm{\hat{x}} \gets \arg \min_{\substack{\bm{0} \leq \bm{x} \leq \bm{1} \atop \bm{x}^\top \bm{1} = 1 } } \frac{1}{n} \bm{x}^\top \bm{\hat{\Omega}}_n \bm{x} + (\bm{x}^\top \bm{\hat{b}}_n)^2 .$$
  \STATE \textbf{return} $\bm{\hat{x}}^\top \bm{g}$.
\end{algorithmic}
\end{algorithm}

\section{Experiments}

\begin{figure*}[ht!]
\centering
\includegraphics[width = 0.82\textwidth, height = 0.45\textheight]{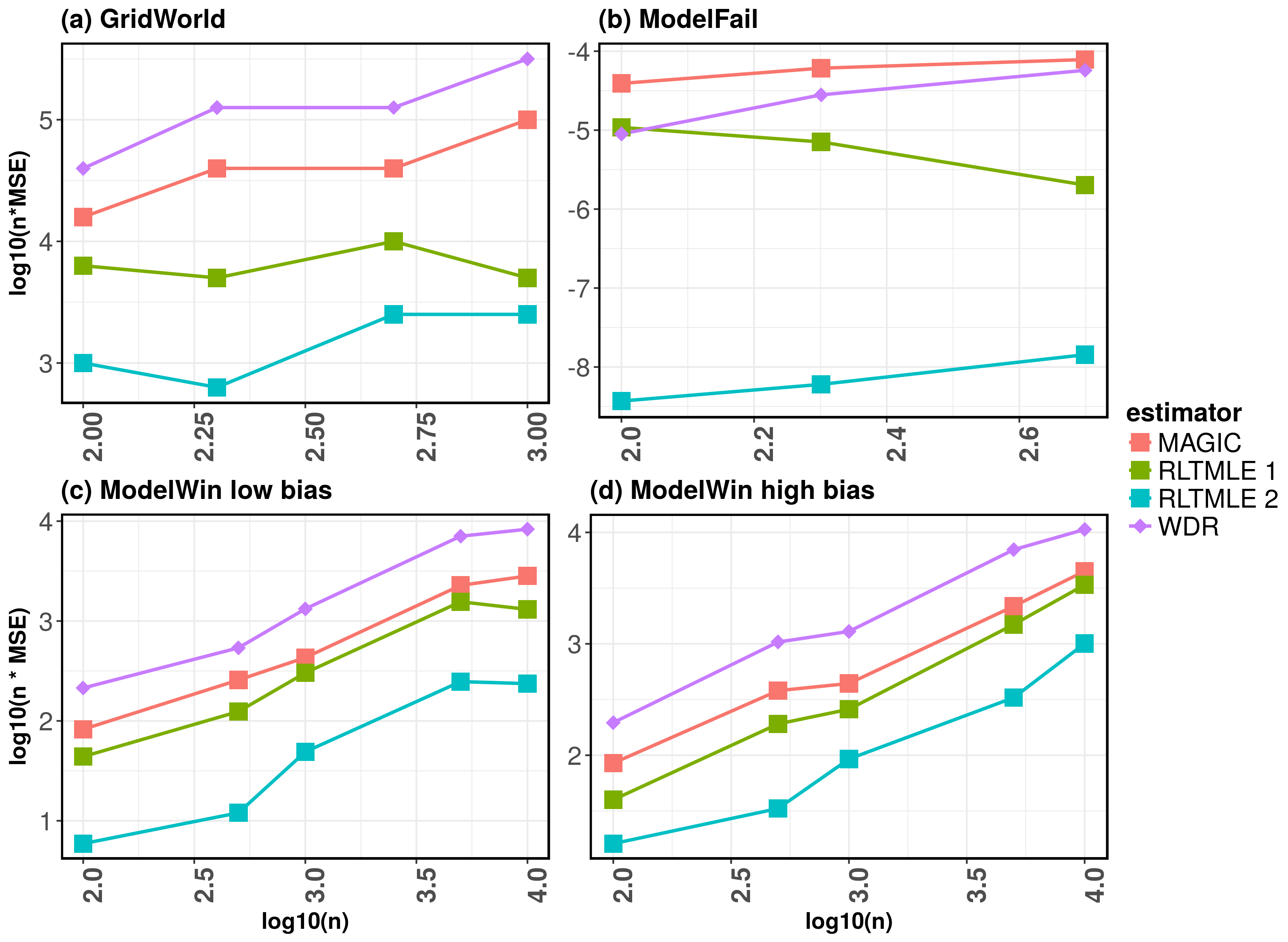}
\caption[]{Empirical results for three different environments and varying level of model misspecification. \textbf{(a)} GridWorld MSE across varying sample size $n=(100, 200, 500, 1000)$ and bias equivalent to $b_0=0.005 \times \text{Normal}(0,1)$ over 71 trials; \textbf{(b)} ModelFail MSE across varying sample size $n=(100, 200, 500, 1000)$ and bias equivalent to $b_0=0.005 \times \text{Normal}(0,1)$ over 71 trials; \textbf{(c)} ModelWin MSE across varying sample size $n=(100, 500, 1000, 5000, 10000)$ and bias equivalent to $b_0=0.005 \times \text{Normal}(0,1)$ over 63 trials; \textbf{(d)} ModelWin MSE across varying sample size $n=(100, 500, 1000, 5000, 10000)$ and bias equivalent to $b_0=0.05 \times \text{Normal}(0,1)$ over 63 trials.}
\label{all_results}
\end{figure*}

\begin{figure*}[ht!]
\centering
\includegraphics[width = 0.82\textwidth, height = 0.27\textheight]{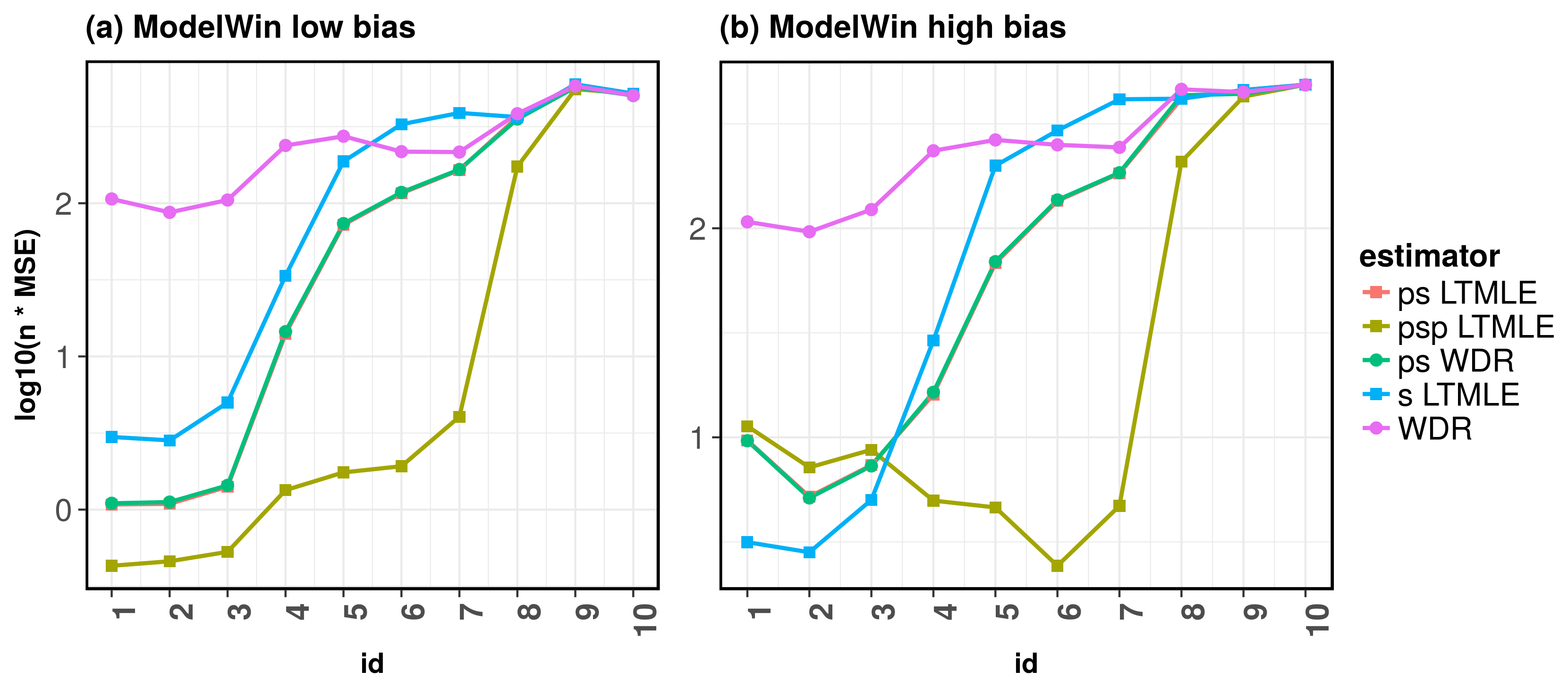}
\caption[]{Comparison of WDR and LTMLE base estimators across various regularization methods in ModelWin at low ($b_0=0.005\times \text{Normal}(0,1)$) and high ($b_0=0.05 \times \text{Normal}(0,1)$) model misspecification.
Regularized base estimators include ps LTMLE (partial, softened LTMLE), ps WDR (partial, softened WDR), psp LTMLE (partial, softened, penalized LTMLE), s LTMLE (softened LTMLE) and WDR (no regularization). The x-axis indicates the id of the $k^{th}$ estimator, corresponding to $(\alpha_k,\lambda_k,\tau_k)$. \textbf{(a)} ModelWin MSE for sample size $n=1000$ and low bias over 315 trials; \textbf{(b)} ModelWin MSE for sample size $n=1000$ and high bias over 315 trials.}
\label{base_est_results}
\end{figure*}

In this section, we demonstrate the effectiveness of RLTMLE by comparing it with other state-of-the-art methods used for OPE problem in various RL benchmark environments. We used three main domains, with detailed description of each allocated to the Appendix. We implement the same behavior and evaluation policies as in previous work \cite{thomas2016, Fara2018}.
\begin{enumerate}
    \item \textbf{ModelFail}: a partially observable, deterministic domain with $T=3$. Here the approximate model is incorrect, even asymptotically, due to three of the four states appearing identical to the agent. 
    \item \textbf{ModelWin}: a stochastic MDP with $T=10$, where the approximate model can perfectly represent the MDP.
    \item \textbf{GridWorld}: a $4 \times 4$ grid used for evaluating OPE methods, with an episode ending at $T=100$ or when a final state ($s16$) is reached. 
\end{enumerate}

We omit benefits of RLTMLE over IS, PDIS (per-decision IS), WIS (weighted IS), CWPDIS (consistent weighted per-decision IS) and DR (doubly robust) estimators due to the extensive empirical studies performed by Thomas and Brunskill \cite{thomas2016}. Instead, we compare our estimator to WDR and MAGIC, as they demonstrate improved performance over all simulations in benchmark RL environments considered \cite{thomas2016}.

In evaluating our estimator, we also explore how various degree of model misspecification and sample size can affect the performance of considered methods. We start with small amount of bias, $b_0=0.005*\text{Normal}(0,1)$, where most estimators should do well. Consequently, we increase model misspecification to $b_0=0.05*\text{Normal}(0,1)$ at the same sample size, and consider the performance of all estimators. In addition, we test sensitivity to the number of episodes in $D$ with $n=\{100, 200, 500, 1000)$ for GridWorld and ModelFail, and $n=\{100, 500, 1000, 5000, 10000)$ for ModelWin.  

In addition, we consider the benefits of adding few regularization techniques as opposed to all three described in subsection \ref{regularization}. In particular, we concentrate on RLTMLE with only weight softening and partial LTMLE (RLTMLE 1) as opposed to using penalized LTMLE as well (RLTMLE 2). The goal of these experiments was to demonstrate the improved performance of our estimator when fully exploiting all the variance reduction techniques in a clever way. The MSE across varying sample size and model misspecification for GridWorld, ModelFail and ModelWin can be found in Figure \ref{all_results}. We can see that RLTMLE 2 outperforms all other estimators for all RL environments and varying levels of model misspecification. 

Finally, we compare WDR and LTMLE base estimators augmented with various regularization methods before the ensemble step in Figure \ref{base_est_results}. In particular, for ModelWin, we look at the MSE of $\hat{V}_{1}^{\pi_e, j}(\epsilon_{n,1})(s_1)$ and $g_k$ for each $k$, where the $k^{th}$ estimator corresponds to regularization $(\alpha_k,\lambda_k,\tau_k)$. Regularized base estimators considered include ps LTMLE (partial, softened LTMLE), ps WDR (partial, softened WDR), psp LTMLE (partial, softened, penalized LTMLE), s LTMLE (softened LTMLE) and WDR (no regularization). We note the vast improvement of WDR just by adding weight softening across all base estimators, evident for both low and high model misspecification setting. For the low bias environment of ModelWin, psp LTMLE (RLTMLE 2) uniformly outperforms all competitors for all $k$. High bias setting loses to s LTMLE for low $k$, but still outperforms majority of the time, including having the best ensemble MSE. While uniform win over all $k$ is not necessary, we note that this behavior stems from the fact that for $k<3$, $(\alpha_k,\lambda_k,\tau_k)$ used had very small $\tau_k$ and $\alpha_k$. As such, with no strong debiasing effect of LTMLE, minimizing variance becomes more effective with respect to minimizing MSE. 

\section{Conclusion}

In this paper, we proposed a new doubly robust estimator for off-policy value evaluation in reinforcement learning. In particular, we present a convex combination of regularized LTMLE estimators which aim at minimizing the MSE. We showed that our estimator is consistent and asymptotically optimal, achieving the Cramer-Rao lower bound. We prove the $O_P(1/\sqrt{n})$ rate of convergence of our estimator, and characterize its asymptotic distribution. The LTMLE is guaranteed to lie in the allowed rewards domain, both for discrete and continuous state, and is amenable to several regularization techniques. Finally, our experiments demonstrate uniform win of RLTMLE over all considered off-policy methods across multiple RL environments and various levels of model misspecification. 

The RLTMLE enjoys multiple distinguishing features that contribute to its finite sample performance. First, its base estimator is a substitution estimator, therefore it inherently respects the reward domain for the RL problem. While this is true for DR if states and actions are discrete, our estimator by design produces estimates that lie in the allowed reward domain for both discrete and continuous state space. Our estimator also allows for clever usage of importance weights, instead of explicitly summing over IS terms. This property strives from using LTMLE as a base estimator, where stabilized IS ratios can be used as weights of the observations in the log likelihood of the second-stage models. This is an important feature of RLTMLE, that greatly contributes to its stability without introducing bias. Finally, LTMLE is amenable to many regularization methods, with RLTMLE enjoying a rich family of regularized base estimators. Our experiments show impressive performance gains from utilizing variance reduction techniques for both RLTMLE and WDR.

Finally, our method does not refit the entire reward-to-go model for each new target policy as the More Robust Doubly Robust estimator, demonstrating some practical advantages. Since refitting the reward-to-go model can be quite computationally expensive, our estimator might be beneficial in situations where one wants to scan through many candidate target policies.

\section*{Acknowledgements}

The authors would like to thank Philip Thomas and the reviewers for valuable comments.

\newpage

\bibliography{RLTMLE}
\bibliographystyle{icml2019}


\newpage 

\appendix

\section{Appendix organization}

This appendix is organized as follows. In section \ref{section_theory_simplified_algorithm}, we prove theorem \ref{theorem_simplified_algorithm}, which characterizes the statistical properties of the simplified algorithm presented in the main text of the article. Although we have also derived a more advanced and efficient version of the LTMLE estimator, which we introduce in section \ref{section_CV_LTMLE} of this appendix, we choose to present the simplified version first, so as to convey the key ideas of the theoretical analysis without burdening our reader with too many technicalities.

In section B, we derive the EIF of the policy value, a necessary preliminary to establishing the semiparametric efficiency of DR estimators.

In section \ref{section_CV_LTMLE}, we present a more advanced version of the LTMLE estimator, which makes better use of the data. This results in a constant factor speed-up of the convergence rate. This more advanced algorithm also relies on sample splitting, but fits each second stage model using the full sample, insteasd of just using a split of the full sample.

\paragraph{Note on notation.} So as to lighten notation, we will drop the $\pi_e$ superscript.

\section{Theoretical analysis of the simplified sample-splitting-based algorithm}\label{section_theory_simplified_algorithm}

In this section, we walk our reader through the theoretical analysis for the algorithm derived in section \ref{section_formal_presentation_simplified_algo}. We outline the steps of the proof in the proof sketch below. We then state the four main lemmas on which our proof relies, and then present the formal proof.

\begin{proof_sketch}
The first fact underpinning our proof is that for any of candidate action-value $Q'=(Q'_1,...,Q'_T)$ and corresponding value functions $V'=(V'_1,...,V'_T)$, the difference between the candidate and the true value function at time point $t=1$ can be decomposed as follows:
\begin{align}
    V'_1(s_1) - V(s_1) = -\int D(Q')(h) dP^{\pi_b}(h), \label{proof_sketch_thm1_eq1}
\end{align}
where $D(Q')(h) = \sum_{t=1}^T D_t(Q')(h)$, with $D_t(Q')(h) = \rho_{1:t}(h)(r_t + \gamma V'_{t+1}(s_{t+1}) - Q'_t(s_t, a_t))$. This is formally stated in lemma \ref{lemma_simple_algo_1st_order_expansion} below. For non-random functions $Q'$ and $V'$ note that the RHS of \eqref{proof_sketch_thm1_eq1} is equal to $-E_{P, \pi_b}[D(Q')]$.

The second fact our proof relies on is that the estimators $\hat{Q}(\epsilon_n)$ resulting from the fitting of the parametric second stages verify the following equation:
\begin{align}
    \frac{1}{n} \sum_{i=1}^n D(\hat{Q}(\epsilon_n))(H_i) = 0. \label{proof_sketch_thm1_eq2}
\end{align}
This is formally stated in lemma \ref{lemma_simple_algo_EIC_eq} below. The argument in the proof of lemma \ref{lemma_simple_algo_EIC_eq} can be simply summarized as follows. For each $t$, $D_t(\hat{Q}(\epsilon_{n,t}))$ is the score function of the log likelihood of the second-stage logistic model for time point $t$.

The third fact we use in our proof is that $\epsilon_{n}$ converges in probability to some limit $\epsilon_\infty$.  Heuristically, the reason why this is the case is that, due to the convergence of $\hat{Q}_n$ to $Q_\infty$, the log likelihoods of the second stage models converge to a limit, which in turns implies that their arg min $\epsilon_n$ converge to the arg min of their limit. We make this rigorous in lemma \ref{lemma_simple_algo_convergnce_epsilon_n} below.

Using the first two facts stated above, we obtain, by adding up equations \eqref{proof_sketch_thm1_eq1} and \eqref{proof_sketch_thm1_eq2}, that the difference between our estimator $\hat{V}_1^{LTMLE}(\epsilon_n)(s_1)$ and the truth $V_1(s_1)$ is
\begin{align}
    &\hat{V}_1^{LTMLE}(\epsilon_n)(s_1) - V_1(s_1) \\
    &= \frac{1}{n} \sum_{i=1}^n D(\hat{Q}(\epsilon_n))(H_i) - \int D(\hat{Q}(\epsilon_n))(h) dP^{\pi_b}(h). \label{proof_sketch_thm1_eq3}
\end{align}
Using the third fact stated above, that $\epsilon_n$ converges to some $\epsilon_{\infty}$, motivates rewriting the above display as
\begin{align}
    &\hat{V}_1^{LTMLE}(\epsilon_n)(s_1) - V_1(s_1) \\
    &= \frac{1}{n} \sum_{i=1}^n D(\hat{Q}(\epsilon_\infty))(H_i) - \int D(\hat{Q}(\epsilon_\infty))(h) dP^{\pi_b}(h) \\
    &+\frac{1}{n} \sum_{i=1}^n D(\hat{Q}(\epsilon_n))(H_i) - D(\hat{Q}(\epsilon_\infty))(H_i)\\
    &- \int D(\hat{Q}(\epsilon_\infty))(h) - D(\hat{Q}(\epsilon_n))(h) dP^{\pi_b}(h).  \label{proof_sketch_thm1_eq4}
\end{align}
Denote $\mathcal{T}$ the sample split on which the initial estimators are fitted. Since $h \mapsto D(\hat{Q}(\epsilon_\infty))(h)$ is a non-random function conditional on $\mathcal{T}$, $\linebreak\int D(\hat{Q}(\epsilon_\infty))(h) dP^{\pi_b}(h) = E_{P, \pi_b}[D(\hat{Q}(\epsilon_\infty))|\mathcal{T}]$. Therefore, applying the Central Limit theorem conditional on $\mathcal{T}$ gives us that the first line of the RHS in the above display is asymptotically normally distributed and is of order $O_P(1/\sqrt{n})$. As we will show in the formal proof, this also holds after marginilazing w.r.t. $\mathcal{T}$.

The term formed by the second and third lines in the RHS of the above display can be shown to be $o_P(1/\sqrt{n)}$. This is formally stated in lemma \ref{lemma_simple_algo_maximal_inequality} below.
\end{proof_sketch}

The following lemma gives a useful decomposition of the difference between any candidate state-value function $V'_1$ and the true state-value function $V_1$. 
\begin{lemma}[First order expansion]\label{lemma_simple_algo_1st_order_expansion}
Consider $Q'=(Q'_1,...,Q'_T)$ a candidate vector of action-value functions $\mathcal{S} \times \mathcal{A} \rightarrow \mathbb{R}$ for polict $\pi_e$, and let $V'=(V_1',...,V'_T)$ the corresponding vector of state-value functions under $\pi_e$, that is, for all $t$, $s\in \mathcal{S}$, $V_t'(s) = \sum_{a'\in \mathcal{A}} \pi_e(a'|s) Q'_t(s,a')$. Denote $Q=(Q_1,...,Q_T)$ and $V=(V_1,...,V_T)$ the true action-value and state value functions under $\pi_e$. For all $t$, for all $h \in \mathcal{H}$, denote $\rho_{1:t}'(h)$ an importance sampling ratio for time point $t$ and trajectory $h$, not necessarily equal to the true importance sampling ratio. Denote $\rho=(\rho_1,...,\rho_T)$ and $\rho'=(\rho'_1,...,\rho'_T)$. We have that
\begin{align}
    V'_1(s_1) - V_1(s_1) = &- \int D(\rho', Q')(h) dP^{\pi_b}(h)\\
    & - \int Rem(\rho, \rho', Q, Q')(h) dP^{\pi_b}(h),
\end{align}
where 
\begin{align}
D(\rho', Q')(h) = \sum_{t=1}^T D_t(\rho', Q')(h)
\end{align} 
and 
\begin{align}
Rem(\rho, \rho', Q, Q')(h)\linebreak = \sum_{t=1}^T Rem_t(\rho, \rho', Q, Q')(h)
\end{align}
with
\begin{align}
    D_t(\rho', Q')(h) = \gamma^{t-1}\rho'_{1:t}(h) \big(&r_t + \gamma V'_{t+1}(s_{t+1})\\
    & - Q'_t(s_t, a_t) \big),
\end{align}
and 
\begin{align}
    &Rem_t(\rho, \rho', Q, Q')(h) \\
    &= \gamma^{t-1} \big(\rho_{1:t}(h) - \rho'_{1:t}(h)\big) \big(Q_t(s_t,a_t) - Q'_t(s_t,a_t) \\
    &\qquad \qquad \qquad+ (V_{t+1}(s_{t+1}) - V'_{t+1}(s_{t+1})) \big) .
\end{align}
From the expression in the RHS of the above display, it is immediately clear that $Rem_t(\rho, \rho', Q, Q')(h)=0$ if $\rho=\rho'$ or $Q = Q'$.
\end{lemma}

The lemma below shows that the maximum likelihood fits $\epsilon_{n,t}$ of the second-stage parametric models solve a certain equation, termed score equation in statistics.
\begin{lemma}[Score equation]\label{lemma_simple_algo_EIC_eq}
Consider the simplified LTMLE algorithm described in section \ref{section_formal_presentation_simplified_algo}.
For each $t=1,...,T$, the maximum likelihood fit $\epsilon_{n,t}$ satisfies
\begin{align}
    \sum_{i=1}^n D_t(\rho_{1:t}, \hat{Q}(\epsilon_{n,t}))(H_i) = 0.
\end{align}
\end{lemma}

The following lemma shows that the vector $\epsilon_n=(\epsilon_{n,1},...,\epsilon_{n,T})$ of the maximum likelihood fits of the second stage models converges in probability to a limit.
\begin{lemma}[Convergence of $\epsilon_n$]\label{lemma_simple_algo_convergnce_epsilon_n}
Make assumptions \ref{assumption_bounded_rewards}, \ref{assumption_convergence_initial_estimator}, \ref{assumption_bounds_limit_initial_estimator} and \ref{assumption_bound_IS_ratios}. Then, there exists $\epsilon_\infty \in \mathcal{R}^T$ such that
\begin{align}
    \epsilon_n - \epsilon_\infty = o_P(1).
\end{align}
\end{lemma}

The following lemma allows to bound the last two lines of the RHS in \eqref{proof_sketch_thm1_eq4} from the proof sketch above.
\begin{lemma}[Equicontinuity]\label{lemma_simple_algo_equicontinuity} 
Denote, for all $h \in \mathcal{H}$, $\epsilon \in \mathbb{R}$, $Q'$ and $\rho'$
\begin{align}
g_\epsilon(Q', \rho')(h) = D(Q'(\epsilon), \rho')(h),
\end{align}
where $Q'$ and $\rho'$ are possibly random. Suppose $H_1,...,H_n$ are i.i.d. trajectories drawn from $P^{\pi_b}$. Suppose further that $H_1,...,H_n$ are independent from the potentially random functions $Q'$ and $\rho'$. Suppose $\epsilon'_n \xrightarrow{P} \epsilon'_\infty$ for some $\epsilon'_\infty$. Then 
\begin{align}
 &\frac{1}{n} \sum_{i=1}^n g_{\epsilon'_n}(Q', \rho')(H_i) - \int g_{\epsilon'_n}(Q', \rho')(h) dP^{\pi_b}(h)\\
 &-\frac{1}{n} \sum_{i=1}^n g_{\epsilon'_\infty}(Q', \rho')(H_i) - \int g_{\epsilon'_\infty}(Q', \rho')(h) dP^{\pi_b}(h) \\
 &= o_P\left(\frac{1}{\sqrt{n}}\right).
\end{align}
\end{lemma}

We now present the formal proof of theorem \ref{theorem_simplified_algorithm}.

\begin{proof}
From lemma \ref{lemma_simple_algo_1st_order_expansion},
\begin{align}
    \hat{V}_1^{TMLE}(s_1) - V_1(s_1) &= -P^{\pi_b} D(\hat{Q}_n(\epsilon_n), \rho).
\end{align}
Since from lemma \ref{lemma_simple_algo_EIC_eq} we have $P_n (D(\hat{Q}_n(\epsilon_n), \rho) =0$, we can add this latter identity to the above display, which yields
\begin{align}
    &\hat{V}_1^{TMLE}(s_1) - V_1(s_1) \\
    =&  (P_n-P^{\pi_b}) D(\hat{Q}_n(\epsilon_n), \rho)\\
    =& (P_n-P^{\pi_b}) D(Q_\infty(\epsilon_\infty), \rho) \label{proof_thm1_eq1}\\
    &+ (P_n-P^{\pi_b}) (D(Q_\infty(\epsilon_\infty), \rho) -D(\hat{Q}_n(\epsilon_n), \rho)). \label{proof_thm1_eq2}
\end{align}
From the Central Limit theorem applied conditionally on $\mathcal{T}$, 
\begin{align}
    &\sqrt{n} ((P_n-P^{\pi_b}) D(Q_\infty(\epsilon_\infty), \rho))\\
    & \xrightarrow{d} \mathcal{N}(0, \sigma^2(Q_\infty(\epsilon_\infty)),
\end{align} with
\begin{align}
    \sigma^2(Q_\infty(\epsilon_\infty)) := Var_{P^{\pi_b}}(D(Q_\infty(\epsilon_\infty), \rho)).
\end{align}
Using dominated convergence on the c.d.f. on the LHS, 
\begin{align}
    &\sqrt{n} ((P_n-P^{\pi_b}) D(Q_\infty(\epsilon_\infty), \rho))\\
    & \xrightarrow{d} \mathcal{N}(0, \sigma^2(Q_\infty(\epsilon_\infty))
\end{align} also holds true unconditionally.
As proven in section B,  the variance of the $D(Q_\infty(\epsilon_\infty), \rho)$ is the efficient variance from the Cramer-Rao lower bound, provided $Q_\infty(\epsilon_\infty)=Q$, that is provided the initial estimator's model is correctly specified. This is the notion of local efficiency from semiparametric statistics \cite{robins1995, tmle2006}.

From lemmas \ref{lemma_simple_algo_convergnce_epsilon_n} and \ref{lemma_simple_algo_equicontinuity}, the line \eqref{proof_thm1_eq2} is $o_P(1 /\sqrt{n}).$

Therefore, we have that
\begin{align}
    \sqrt{n} (\hat{V}_1^{TMLE}(s_1) - V_1(s_1)) \xrightarrow{d} \mathcal{N}(0, \sigma^2(Q_\infty(\epsilon_\infty)),
\end{align}
and that 
\begin{align}
    E_{P^{\pi_b}} \left[\hat{V}_1^{TMLE}(s_1) - V_1(s_1)\right] = o(1/\sqrt{n}).
\end{align}
\end{proof}

\subsection{Proof of lemma \ref{lemma_simple_algo_1st_order_expansion}}

\begin{proof}
Let $H \sim P^{\pi_e}$. If $Q'$, $V'$ are random functions, further suppose, without loss of generality, that $H$ is independent of $Q'$ and $V'$. Denote $\mathcal{G}$ a $\sigma$-field such that $Q',V'$ are $\mathcal{G}$-measurable.

\paragraph{Step 1.} Observe that 
\begin{align}
    P^{\pi_b} D(Q', \rho') = P^{\pi_b} D(Q', \rho) + P^{pi_b} (D(Q', \rho') - D(Q', \rho)).
\end{align}

\paragraph{Step 2: First order term.} Observe that 
\begin{align}
    P^{\pi_b} D(Q',\rho) = E_{P^{\pi_b}}[D(Q', \rho)(H)|\mathcal{G}].
\end{align}
For all $t \geq 1,...,T$, denote $\mathcal{F}_t$ the $\sigma$-field induced by $S_1,A_1,R_1,...,S_t,A_t,R_t$. Observe that
\begin{align}
    &E_{P^{\pi_b}}[D_t(Q', \rho)(H)|S_t, A_t, \mathcal{F}_{t-1}, \mathcal{G}]\\
    &= \gamma^{t-1} E_{P^{\pi_b}} [\rho_{1:t} (R_t + \gamma V'_{t+1}(S_{t+1}) \\
    &\qquad \qquad - Q'_t(S_t, A_t)) |S_t, A_t, \mathcal{F}_{t-1}, \mathcal{G}]\\
    &= \gamma^{t-1} \rho_{1:t} E_P [ R_t + \gamma V_{t+1}(S_{t+1})\\
    &\qquad \qquad - Q'_t(S_t, A_t) |S_t, A_t, \mathcal{G}]\\
    &\  + \gamma^t \rho_{1:t} E_P[(V'_{t+1}(S_{t+1})-V_{t+1}(S_{t+1})) |S_t, A_t, \mathcal{G}].
\end{align}
Recall that by definition of $Q$, we have that $E_P[R_t + \gamma V_{t+1}(S_{t+1})| S_t, A_t] = Q_t(S_t,A_t)$. Inserting this in the last line of the above display yields
\begin{align}
     &E_{P^{\pi_b}}[D_t(Q', \rho)(H)|S_t, A_t, \mathcal{F}_{t-1}, \mathcal{G}] =\\
     & \gamma^{t-1} \rho_{1:t} (Q_t(S_t,A_t) -Q'_t(S_t,A_t)) \\
     &+ \gamma^t \rho_{1:t} E_P[V'_{t+1}(S_{t+1}) - V_{t+1}(S_{t+1})|S_t,A_t,\mathcal{G}]. \label{lemma_1st_order_exp_id1}
\end{align}
We take the expectation conditional on $S_t$, $\mathcal{F}_{t-1}$, $\mathcal{G}$ of the first term in the right-hand side of the above display:
\begin{align}
    &E_{P^{\pi_b}}[\gamma^{t-1} \rho_{1:t} (Q_t(S_t, A_t) -Q'_t(S_t, A_t))|S_t, \mathcal{F}_{t-1}, \mathcal{G}] \\
    &= \gamma^{t-1} \rho_{1:t-1} E_{P, \pi_b}[\rho_t (Q_t(S_t,A_t) -Q'_t(S_t,A_t) |S_t, \mathcal{G}] \\
    &= \gamma^{t-1} \rho_{1:t-1} E_{P, \pi_e}[(Q_t(S_t, A_t) - Q'_t(S_t, A_t) | S_t, \mathcal{G}]\\
    &= \gamma^{t-1}\rho_{1:t-1} (V_t(S_t) - V'_t(S_t)).\label{lemma_1st_order_exp_id2}
\end{align}
The second equality above uses that, for all $\mathcal{G}$-measurable function $f$, 
\begin{align}
E_{P, \pi_b}[\rho_t f(S_t,A_t)|S_t,\mathcal{G}] = E_{P, \pi_e}[f(S_t,A_t) | S_t, \mathcal{G}].
\end{align}
The third equality follows from the relationship between the value function and the action value function.

Using the law of iterated expectations, and identities \eqref{lemma_1st_order_exp_id1} and \eqref{lemma_1st_order_exp_id2} yields
\begin{align}
    & E_{P, \pi_b}[D_t(Q', \rho)(H)|\mathcal{G}] \\
    &=  E_{P, \pi_b}[E_{P, \pi_b}[D_t(Q', \rho)(H)|S_t, A_t,\mathcal{F}_{t-1}, \mathcal{G}] |\mathcal{G}]\\
    &= E_{P, \pi_b}[ \gamma^{t-1} \rho_{1:t}(Q_t(S_t,A_t) - Q'_t(S_t, A_t)) | \mathcal{G}]\\
    &\ + E_{P, \pi_b}[\gamma^t \rho_{1:t}E_P[V'_{t+1}(S_{t+1}) - V_{t+1}(S_{t+1}) | S_t, A_t, \mathcal{G} ] | \mathcal{G}]\\
    &= E_{P, \pi_b}[E_{P, \pi_b}[ \gamma^{t-1} \rho_{1:t}(Q_t(S_t,A_t) \\
    &\qquad \qquad \qquad \qquad - Q'_t(S_t, A_t)) | S_t, \mathcal{F}_{t-1}, \mathcal{G}] | \mathcal{G}]\\
    &\ + E_{P, \pi_b}[\gamma^t \rho_{1:t}(V'_{t+1}(S_{t+1}) - V_{t+1}(S_{t+1})) | \mathcal{G}  ]\\
    &= E_{P, \pi_b}[\gamma^{t-1} \rho_{1:t-1}(V_t(S_t) - V'_t(S_t))|\mathcal{G}]\\
    &\ + E_{P, \pi_b}[\gamma^t \rho_{1:t}(V'_{t+1}(S_{t+1}) - V_{t+1}(S_{t+1})) | \mathcal{G}  ] \label{lemma_1st_order_exp_id3}
\end{align}
Using the above expression in the definition of $D(Q',V')$ yields
\begin{align}
    &E_{P, \pi_b}[D(Q',\rho)(H)|\mathcal{G}] \\
    =& \sum_{t=1}^T E_{P, \pi_b}[\gamma^t \rho_{1:t}(V'_{t+1}(S_{t+1}) -V'_{t+1}(S_{t+1})) \\
    &\qquad \qquad \qquad- \gamma^{t-1} \rho_{1:t-1} (V'_t(S_t) - V_t(S_t))|\mathcal{G}]\\
    =& E_{P, \pi_b}[\gamma^T \rho_{1:T+1} (V'_{T+1}(S_{T+1}) - V_{T+1}(S_{T+1}))\\
    &\ -\rho_{1:0}(V'_1(s_1) - V_1(s_1))|\mathcal{G}]\\
    =& -(V'_1(s_1) - V_1(s_1)),
\end{align}
where we have used that by convention $V'_{T+1}(S_{T+1}) = V_{T+1}(S_{T+1}) = 0$ and $\rho_{1:0}=1$.

\paragraph{Step 3: remainder term.} We similarly show that $P^{\pi_b} (D'(Q', \rho) - D(Q', \rho)) = Rem_t(Q, Q', \rho, \rho')$.
\end{proof}

\subsection{Proof of lemma \ref{lemma_simple_algo_EIC_eq}}

We present a proof sketch in this subsection. The complete formal proof is presented in the case of the full algorithm in section B.

\begin{proof_sketch}
The result essentially follows from the following two facts:
\begin{itemize}
    \item The score of the logistic likelihood of the second stage model for time point $t$ is $P_n D_t(\hat{Q}, \rho)$,
    \item A maximum likelihood fit solves the empirical score equation.
\end{itemize}
\end{proof_sketch}

\subsection{Proof of lemma \ref{lemma_simple_algo_convergnce_epsilon_n}}

We present a proof sketch in this subsection. The complete formal proof is presented in the case of the full algorithm in section B.

\begin{proof_sketch}
The convergence of $\hat{Q}$ to $Q_\infty$ implies the pointwise convergence of the log likelihood risk $\mathcal{R}_{n,t}$ to some asymptotic risk $\mathcal{R}_{\infty, t}$. The fact that $Q_{\infty,t} \in [\delta, 1 - \delta] \subset (0,1)$ (in other words, that $Q_{\infty,t}$ is bounded away from $0$ and $1$) implies that the asymptotic log likelihood risk $\mathcal{R}_{\infty, t}$ is strongly convex. This implies it has a unique minimizer $\epsilon_{\infty, t}$. We then show in the formal proof that since $\mathcal{R}_{n,t}$ are a sequence of convex functions that converge pointwise in probability to a strongly convex function minimized by $\epsilon_{\infty, t}$, the sequence of their minimizers $\epsilon_{n,t}$ converges in probability to $\epsilon_{\infty, t}$
\end{proof_sketch}

\subsection{Proof of lemma \ref{lemma_simple_algo_equicontinuity}}

The proof of lemma \ref{lemma_simple_algo_equicontinuity} relies on the following three technical lemmas.
Recall the following definition: for all $Q'$ $\rho'$, $h\in \mathcal{H}$, $\epsilon \in \mathbb{R}$,
\begin{align}
g_\epsilon(Q', \rho')(h) = D(Q'(\epsilon), \rho')(h).
\end{align}

\begin{lemma} \label{lemma_bound_and_Lipschitzness_EIF}
Assume that $0 \leq \rho'_{1:t}(H) \leq M$ almost surely for all $t=1,...,T$. Make assumption \ref{assumption_bounded_rewards} on the range of the rewards. Then for all $\epsilon \in \mathbb{R}^T$, \begin{align}
    \|g_\epsilon(Q', \rho') \|_{L_\infty(P^{\pi_b})} \leq 3MT,
\end{align}
and for all $\epsilon_1, \epsilon_2 \in \mathbb{R}^T$
\begin{align}
    \| g_{\epsilon_1}(Q', \rho') - g_{\epsilon_2}(Q', \rho')\|_{L_\infty(P^{\pi_b})} \leq 2MT \|\epsilon_1 - \epsilon_2\|_\infty.\label{lemma_bound_and_Lipschitzness_EIF_Lipschitz_eq}
\end{align}
\end{lemma}

For any $\epsilon_0 \in \mathbb{R}$, and any $\xi > 0$, define the class of functions
\begin{align}
    &\mathcal{G}(Q', \rho')(\epsilon_0, \xi) \\
    &:= \{ g_\epsilon(Q', \rho') - g_{\epsilon_0}(Q', \rho') : \|\epsilon - \epsilon_0\|_\infty \leq \xi \}.
\end{align}
The next lemma characterizes covering numbers of this class of functions. Covering numbers are a measure of geometric complexity whose definition we recall here (we reproduce the definition 2.1.6. from \cite{vdV-Wellner1996}).

\begin{definition}[Covering number] The covering number $N(\epsilon, \mathcal{F}, \|\cdot\|)$ is the minimal number of balls $\{g:\|f-g\| \leq \epsilon \}$ of radius $\epsilon$ needed to cover the set $\mathcal{F}$.

\end{definition}

\begin{lemma}\label{lemma_covering_number} For any $\alpha > 0$, for any probability distribution $\Lambda$ on $\mathcal{H}$,
\begin{align}
    N(\alpha, L_2(\Lambda), \mathcal{G}(\epsilon_0, \xi)) \leq \left(\frac{2 \xi L}{ \alpha} \right)^T,
\end{align}
with $L=2MT.$
\end{lemma}

\begin{proof}
Consider the set 
\begin{align}
    \bigg\{ & \left(\epsilon_{0,1} + i_1 \frac{\alpha}{L},..., \epsilon_{0,T} + i_T \frac{\alpha}{L}\right) \\
    &: \forall t=1,...T,\  i_t \in \mathbb{Z} \cap  \left[-\frac{\xi L}{\alpha}, \frac{\xi L}{\alpha}\right] \bigg\}. \label{proof_lemma_covering_number}
\end{align}
Observe that for any $f_\epsilon := g_\epsilon(Q', \rho') - g_{\epsilon_0}(Q', \rho') \in  \mathcal{G}(Q', \rho')(\epsilon_0, \xi)$, there exists an $f_{\epsilon'} := g_{\epsilon'}(Q', \rho') - g_{\epsilon_0}(Q', \rho')$ in the set above such that $\|\epsilon - \epsilon'\|_\infty \leq \alpha / L$.
From the second claim in lemma \ref{lemma_bound_and_Lipschitzness_EIF}, for all $h \in \mathcal{H}$, $|f_{\epsilon'}(h) - f_\epsilon(h)| \leq \alpha$. Therefore, for any probability distribution $\Lambda$ over $\mathcal{H}$, 
\begin{align}
    \|f_{\epsilon'} - f_\epsilon\|_{L_2(\Lambda)} =& \left( \int (f_{\epsilon'}(h) - f_\epsilon(h))^2 d\Lambda(h) \right)^{1/2}\\
    \leq & \alpha.
\end{align}
Therefore the set defined above is an $\alpha$-cover of $\mathcal{G}(\epsilon_0, \xi))$ for the norm $L_2(\Lambda)$. Since this set has at most $(2 \epsilon_L / \alpha)^T$ elements, this proves that
\begin{align}
    N(\alpha, L_2(\Lambda), \mathcal{G}(\epsilon_0, \xi)) \leq \left(\frac{2 \xi L}{ \alpha} \right)^T.
\end{align}
\end{proof}

The covering numbers characterized in lemma \ref{lemma_covering_number} are the basis for another measure of geometric complexity of a class of function, the uniform entropy integral, whose definition we recall below (see also \cite{vdV-Wellner1996}).

\begin{definition}[Uniform entropy integral] Consider a class of functions $\mathcal{X} \rightarrow \mathbb{R}$. Let $F:\mathcal{X} \rightarrow \mathbb{R}$ be an envelope function for $\mathcal{F}$, that is a function such that for all $x \in \mathcal{X}$, $|f(x)| \leq F(x)$. The uniform entropy integral of $\mathcal{F}$, w.r.t. the envelope function $F$ and $L_2$ norm is defined, for all $\beta > 0$ as 
\begin{align}
    &J_F(\beta, \mathcal{F}, L_2)\\
     := & \int_0^\beta \sup_{\Lambda} \sqrt{\log(1 + N(\alpha \|F\|_{\Lambda,2}, L_2(\Lambda), \mathcal{F})} d\alpha,
\end{align}
where the supremum is over all discrete probability distributions on $\mathcal{X}$.
\end{definition}

The following lemma characterizes the uniform entropy integral of $\mathcal{G}(\epsilon_0, \xi)$.

\begin{lemma}\label{lemma_uniform_entropy_integral} Let $\beta > 0$. Denote $L=2MT$. The function $F_{\xi}:h \mapsto L \xi$ is an envelope function for $\mathcal{G}(\epsilon_0, \xi)$. The uniform entropy integral of $\mathcal{G}(\epsilon_0, \xi)$ w.r.t. the envelope function $F_\xi$ and for the $L_2$ norm is upper bounded as follows:
\begin{align}
    J_{F_\xi}(\beta, \mathcal{G}(\epsilon_0, \xi), L_2) = O\left(T \beta \sqrt{\log(1/\beta)}\right).
\end{align}
\end{lemma}

\begin{proof}\label{lemma_Pollard_maximal_inequality}
For every probability distribution $\Lambda$ on $\mathcal{H}$, $\|F_\xi\|_{\Lambda, 2} = L \xi$. From lemma \ref{lemma_covering_number}, $N(\alpha \|F_\xi\|_{2, \Lambda}, L_2(\Lambda), \mathcal{G}(\epsilon_0, \xi)) \leq (2 / \alpha)^T$. Therefore,
\begin{align}
     J_{F_\xi}(\beta, \mathcal{G}(\epsilon_0, \xi), L_2) \leq & \int_0^\beta \sqrt{\log(1 + (2 / \alpha)^T)} d\alpha \\
     = & O\left(T \beta \sqrt{\log(1 /\beta)}\right),
\end{align}
where the second equality above follows from an integration by parts.
\end{proof}

Finally, we prove the lemma \ref{lemma_simple_algo_equicontinuity}. The proof relies on a classical result in empirical process theory. We first introduce the relevant definitions and the relevant result before stating the proof of our lemma.

\begin{definition}[Empirical process and empirical process notation]
Consider $\mathcal{X}, \Sigma, P')$ a probability space and let $X_1,...,X_n$ be $n$ i.i.d. draws from $P'$. Let $\mathcal{F}$ be a class of functions $\mathcal{X} \rightarrow \mathbb{R}$. 
For all $f \in \mathcal{F}$, define the so-called ``empirical process notation''
\begin{align}
    P'f := \int f(h) dP'(h).
\end{align}
Denote $P_n := n^{-1} \sum_{i=1}^n \delta_{X_i}$ the empirical probability distribution associated to the sample $X_1,...,X_n$. Observe that using the empirical process notation defined above, we have that $P_n f = n^{-1}\sum_{i=1}^n f(X_i)$.
The stochastic process 
\begin{align}
    \{(P_n - P')f : f \in \mathcal{F} \}
\end{align}
is termed the empirical process associated to $P'$ and $n$ indexed by $\mathcal{F}$.
\end{definition}

We restate here the classical empirical process result \cite{vdV-Wellner1996} we will use to prove lemma \ref{lemma_simple_algo_equicontinuity}. (This is lemma 2.14.1 in \cite{vdV-Wellner1996}, for $p=1$ in their notation.)

\begin{lemma}[Pollard's maximal inequality, vdV-Wellner 1996 2.14.1]\label{lemma_Pollard_maximal_inequality}
Consider $\mathcal{X}, \Sigma, P')$ a probability space and let $X_1,...,X_n$ be $n$ i.i.d. draws from $P'$. Let $\mathcal{F}$ be a class of functions $\mathcal{X} \rightarrow \mathbb{R}$. Let $\mathcal{F}$ be a class of functions $\mathcal{X} \rightarrow \mathbb{R}$ with envelope function $F$. Then
\begin{align}
    E_{P'} [\sup_{f \in \mathcal{F}} \sqrt{n} |(P_n - P') f| ] \lesssim J_F(1, \mathcal{F}, L_2) \|F\|_{L_2(P')}.
\end{align}
\end{lemma}

We now have all the ingredients to prove lemma \ref{lemma_simple_algo_equicontinuity}.

\begin{proof}[Proof of lemma \ref{lemma_simple_algo_equicontinuity}]
Recasting the claim of lemma \ref{lemma_simple_algo_1st_order_expansion} in terms of empirical process notation, we want to show that
\begin{align}
    \sqrt{n} (P_n - P^{\pi_b}) (g_{\epsilon_n}(Q', \rho') - g_{\epsilon_\infty}(Q', \rho')) = o_P(1).
\end{align}
Let $\kappa > 0$, $\gamma \in (0,1/2)$. Define, for all $\xi > 0$, the following two events:
\begin{align}
    \mathcal{E}_1(\xi) &:= \left\lbrace \|\epsilon_n  - \epsilon_\infty\|_\infty \leq \xi  \right\rbrace
\end{align}
and
\begin{align}
    &\mathcal{E}_2(\xi) := \\
    &\bigg\{ \sup_{\substack{\epsilon \\ \|\epsilon - \epsilon_\infty\|_\infty \leq \xi }} \sqrt{n} |(P_n - P^{\pi_b}) (g_{\epsilon}(Q', \rho') - g_{\epsilon_\infty}(Q', \rho'))| \\
    & \qquad \leq \kappa \bigg\}.
\end{align}
The function $F_\xi: h \mapsto \xi L$ is an envelope function for $\mathcal{G}(\epsilon_0, \xi)$. By Markov's inequality and lemma \ref{lemma_Pollard_maximal_inequality} applied with the uniform entropy integral bound given in lemma \ref{lemma_uniform_entropy_integral}, we have that
\begin{align}
    &1 - P^{\pi_b}[\mathcal{E}_2(\xi)] \\
    =& P^{\pi_b}\left[ \sqrt{n} |(P_n - P^{\pi_b}) (g_{\epsilon_n}(Q', \rho') - g_{\epsilon_\infty}(Q', \rho'))| \geq \kappa\right]\\
    \leq& \kappa^{-1} E_{P^{\pi_b}} \left[\sqrt{n} |(P_n - P^{\pi_b}) (g_{\epsilon_n}(Q', \rho') - g_{\epsilon_\infty}(Q', \rho'))| \right]\\
    \leq& \kappa^{-1} J_F(1, \mathcal{G}(\epsilon_0, \xi), L_2) \|F_\xi\|_{2, \Lambda}\\
    \leq& K \kappa^{-1} \xi L,
\end{align}
for some constant $K$.
Set $\xi = \kappa \gamma / (2 K L)$. Then, from the above display $P^{\pi_b}[\mathcal{E}_2(\kappa \gamma / (2 K L))] \geq 1 - \gamma / 2$.

Besides, since $\epsilon_n \xrightarrow{P} \epsilon_\infty$, there exists $n_0$ such that for all $n \geq n_0$, $\linebreak P^{\pi_b}[\mathcal{E}_1(\kappa \gamma / (2 K L))] \geq 1 - \gamma / 2$. 
Observe that if $\mathcal{E}_1(\kappa \gamma / (2 K L)) \cap \mathcal{E}_2(\kappa \gamma / (2 K L))$ is realized, then
\begin{align}
    \sqrt{n} |(P_n - P^{\pi_b}) (g_{\epsilon_n}(Q', \rho') - g_{\epsilon_\infty}(Q', \rho'))| \leq \kappa.
\end{align}
Using a union bound, we have that, for all $n \geq n_0$,
\begin{align}
    &P^{\pi_b} \left[\sqrt{n} |(P_n - P^{\pi_b}) (g_{\epsilon_n}(Q', \rho') - g_{\epsilon_\infty}(Q', \rho'))| \leq \kappa \right]\\
    \geq & 1 - (1 -  P^{\pi_b}[\mathcal{E}_1(\kappa \gamma / (2 K L))])\\
     &- (1- P^{\pi_b}[\mathcal{E}_2(\kappa \gamma / (2 K L))])\\
    \geq &  1  - \gamma.
\end{align}
Recapitulating the above, we have proven that for all $\kappa > 0$, $\gamma \in (0,1/2)$, there exists $n_0$ such that for all $n\geq n_0$,
\begin{align}
    & P^{\pi_b} \left[\sqrt{n} |(P_n - P^{\pi_b}) (g_{\epsilon_n}(Q', \rho') - g_{\epsilon_\infty}(Q', \rho'))| \leq \kappa \right]\\
     \geq & 1 - \gamma.
\end{align}
In other words, we have thus proven that 
\begin{align}
    \sqrt{n} |(P_n - P^{\pi_b}) (g_{\epsilon_n}(Q', \rho') - g_{\epsilon_\infty}(Q', \rho'))| = o_P(1).
\end{align}
which concludes the proof.
\end{proof}

\section{Efficiency and efficient influence function derivation}\label{EIF-derivation}

In this section we show that our estimator is optimal in a certain sense. Specifically, we show that it is \textit{locally semiparametric efficient}.
We will introduce our reader to the notions from semiparametric statistics necessary to understand and prove \textit{semiparametric efficiency} of an estimator. 

In particular we will introduce the concept of \textit{efficient influence function} (EIF). Deriving the EIF is the cornerstone of the efficiency analysis. It is also key to the construction of the estimator: in semiparametric statistics, looking for the EIF is typically the starting point for building an efficient estimator.

Deriving the EIF in the general MDP setting is one of the main contributions of this work. 

Note that the presentation of the notions of semiparametric inference is heavily drawn from \cite{Kosorok2006}, and entails no novel contribution of our part. We wrote it so as to make this appendix a self-contained document for the reader non-familiar with semiparametric statistics.

%

\subsection{Introducing notions of optimality from semiparametric statistics}

The notions of optimality we are about to introduce are relative to both the estimand and the statistical model. The statistical model $\mathcal{M}$ is the set of probability distributions we believe to contain the true data-generating mechanism, which we will denote $P_0$. We will typically denote $P$ an arbitrary element of $\mathcal{M}$. The larger the model, the more realistic it is that it contains the truth, but also the larger is the variance of estimators over this model. 

The first notion of optimality we introduce is the notion of \textit{efficiency} \textit{[cite Kosorok]}. An estimator is \textit{efficient} at $P$, if, were the true data-generating mechanism to be $P$, it would have the lowest variance among a certain class of estimators, namely the class of estimators that are \textit{regular} at $P$ w.r.t $\mathcal{M}$. We define formally the notion of a regular estimator at $P$ w.r.t. $\mathcal{M}$ below. The concepts of \textit{regularity}, \textit{efficiency} and \textit{semiparametric efficiency} that we are about to introduce are defined relative to $P$ and to $\mathcal{M}$, but they really only involve $\mathcal{M}$ through its geometry in a neighborohood around $P$. Even more so, these notions only involve $\mathcal{M}$ through its so-called \textit{tangent space} at $P$, which we will denote $T_{\mathcal{M}}(P)$.

We now proceed to stating the formal definitions. In all of this section, we will assume for simplicity that all probability distributions are dominated by the same measure $\mu$. Definitions will be stated accordingly, but our reader should be aware that some of them can be extended to the case where there is not a single dominating measure $\mu$.

\begin{definition}[Statistical model]
A statistical model $\mathcal{M}$ is a collection of probability distributions $\{ P \in \mathcal{M} \}$ on a sample space $\mathcal{X}$.
\end{definition}

Usually, we suppose that the true data-generating mechanism, that we will denote $P_0$ in this section, belongs to the statistical model.

In the reinforcement learning setting of this article, the statistical model is a collection of probability distributions over the space of trajectories $\mathcal{H}$. Any probability distribution $P$ over $\mathcal{H}$ can be factored as follows:
\begin{align}
P = \prod_{t=1}^T \tilde{Q}_t \prod_{t=1}^T \pi_t \equiv \tilde{Q} \pi,
\end{align}
with $\tilde{Q} \equiv \prod_{t=1}^T \tilde{Q}_t$ and $\pi \equiv \prod_{t=1}^T \pi_t$, and where $\tilde{Q}_t$ is the conditional distribution of $R_t, S_{t+1}$ given $S_t, A_t$ and $\pi_t$ is the conditional distribution of $A_t$ given $S_t$. Since we know the logging policy, our statistical model supposes that for any of its elements $P$, $\pi$ is equal to the known value of the logging policy. 
We therefore write our statistical model as indexed by the known value of the logging policy:
\begin{align}
&\mathcal{M}(\pi) = \\
& \left\lbrace P = \prod_{t=1}^T \tilde{Q}_t \prod_{t=1}^T \pi_t : \forall t=1,...,T,\ \tilde{Q}_t \in \mathcal{M}_{Q,t} \right\rbrace\\
=& \left\lbrace P = \tilde{Q} \pi : \tilde{Q} \in \mathcal{M}_Q \right\rbrace,
\end{align}
with $\mathcal{M}_Q = \mathcal{M}_{Q,1} \times ... \times \mathcal{M}_{Q,T}$. We suppose that the model is fully nonparametric, that is for every $t$, $\mathcal{M}_{Q,t}$ is equal to the set of all conditional probability distributions $P_{R_t, S_{t+1}|S_t, A_t}$.

\begin{definition}[Estimand / target parameter]
The target parameter mapping, which we denote $\Psi$, is a map defined on the statistical model $\mathcal{M}$, with values either in $\mathbb{R}^d$ for a certain $d$, or in some function space. The estimand or target parameter is this map evaluated at the data-generating distribution: $\Psi(P_0)$.
\end{definition}

In the setting of this article, for every probability distribution $P = \tilde{Q} \pi$ over $\mathcal{H}$, the target parameter mapping at $P$ is defined as $\Psi(P) = E_{P' \equiv \tilde{Q} \pi_e}[\sum_{t=1}^T \gamma^t R_t ]$. Note that this expression doesn't depend on $\pi$, so we can write $\Psi(P) = \tilde{\Psi}(\tilde{Q})$ for a mapping $\tilde{\Psi}$ thus defined.

\begin{definition}[Estimator]
For any sample size $n$, an estimator $\hat{\Psi}_n$ is a mapping of the sample space $\mathcal{H}^n$ to the space of the target parameter/estimand.
\end{definition}

\begin{definition}[One-dimensional submodel]
A one-dimensional submodel of $\mathcal{M}$ that passes through $P$ in $0$ is a subset of $\mathcal{M}$ of the form $\{P_\epsilon : \epsilon \in [-\eta, \eta] \}$, for some $\eta > 0$, such that $P_{\epsilon=0} = P$.
\end{definition}

\begin{definition}[Score of a one-dimensional parametric model]
The score in $\epsilon_0$, for any $\epsilon_0$, of a one-dimensional parametric model $\{P_\epsilon : \epsilon \}$ is the derivative of the log-likelihood w.r.t. $\epsilon$, evaluated at $\epsilon_0$. Denoting it $s$, the score is the function defined, for all $x$ in the sample space $\mathcal{X}$,
\begin{align}
s(x) = \frac{d \log (dP_\epsilon / d \mu)(x)}{d\epsilon}\bigg|_{\epsilon=\epsilon_0}.
\end{align}
\end{definition}

\begin{definition}[Tangent space]
The tangent space of a statistical model $\mathcal{M}$ at $P$, which we denote $T_{\mathcal{M}}(P)$ is the linear closure of the set of score functions of all of the one-dimensional submodels of $\mathcal{M}$ that pass through $P$. Formally,
\begin{align}
T_{\mathcal{M}}(P) = \overline{\text{Span}}(\mathcal{S}(\mathcal{M},P)),
\end{align}
with
\begin{align}
& \mathcal{S}(\mathcal{M},P) \equiv \\
\bigg\{\frac{d \log (dP_\epsilon / d \mu)}{d\epsilon}\bigg|_{\epsilon=0} : &\{P_\epsilon:\epsilon \} \text{ 1-dim. submodel of } \mathcal{M}, \\
&P_{\epsilon=0}=P \bigg\}.
\end{align}
\end{definition}

\begin{definition}[Regularity]
Suppose the data-generating distribution is $P$, for some $P \in \mathcal{M}$. An estimator is regular at $P$ w.r.t. $\mathcal{M}$ if, for any one dimensional submodel $\{P_\epsilon: \epsilon\}$ of $\mathcal{M}$ such that $P_{\epsilon=0}=P$, the asymptotic distribution of
\begin{align}
\sqrt{n} (\hat{\Psi}_n - \Psi(P_{\epsilon=1/\sqrt{n}}))
\end{align}
is the same as the asymptotic distribution of
\begin{align}
\sqrt{n} (\hat{\Psi}_n - \Psi(P)).
\end{align}
\end{definition}

It is the understanding of the authors of this article that non-regular estimators at $P$ correspond either to pathological estimators or pathological $P$'s.

\begin{definition}[Efficiency]
An estimator is a locally efficient estimator of $\Psi(P)$ at $P$ w.r.t. $\mathcal{M}$ if it has smallest asymptotic variance among all regular estimators of $\Psi(P)$ at $P$ w.r.t. $\mathcal{M}$.
\end{definition}

\begin{definition}[Generalized Cramer-Rao lower bound]
Consider a one-dimensional model $\{P_\epsilon: \epsilon\}$ such that $P_{\epsilon=0} = P$. From classical parametric statistics theory, an regular estimator $\Psi_n$ of $\Psi(P)$ w.r.t. $\{P_\epsilon: \epsilon\}$ has asymptotic variance greater than the Cramer-Rao lower bound $v_{CR}(\{P_\epsilon : \epsilon \})$:
\begin{align}
\lim_{n \rightarrow \infty} n \text{Var}_P(\hat{\Psi}_n) \geq v_{CR}(\{P_\epsilon : \epsilon \})  \equiv \frac{\left(\frac{d\Psi(P_\epsilon)}{d\epsilon}\big|_{\epsilon=0}\right)^2}{\mathcal{I}(P)}
\end{align}
where $\mathcal{I}(P) = {E_P\left[\frac{d^2 \log (dP/d\mu)(X)}{d\epsilon^2} \big|_{\epsilon=0} \right]}$ is the Fisher information of the model $\{P_\epsilon : \epsilon \}$ at $\epsilon=0$.

For a statistical model $\mathcal{M}$, the \textit{generalized Cramer-Rao} lower bound $v_{GCR}$ for $\Psi(P)$ w.r.t. $\mathcal{M}$, is the sup  of the Cramer-Rao lower bound over the parametric submodels of $\mathcal{M}$ through $P$:
\begin{align}
v_{GCR}(\mathcal{M}) \equiv \sup_{\{P_\epsilon : \epsilon \} \subseteq \mathcal{M}, P_{\epsilon=0} = P} v_{CR}(\{P_\epsilon : \epsilon \}).
\end{align}
A parametric submodel whose Cramer-Rao lower bound is equal to the generalized Cramer-Rao lower bound is called a least-favorable parametric submodel.
\end{definition}

\begin{definition}[Semiparametric efficiency] An estimator $\hat{\Psi}_n$ is a \textit{locally semiparametric efficient} of $\Psi(P)$ w.r.t. $\mathcal{M}$ if it is consistent for $\Psi(P)$ and if its asymptotic variance is equal to the generalized Cramer-Rao lower bound, that is if
\begin{align}
\lim_{n \rightarrow \infty} n \text{Var}_P(\hat{\Psi}_n) = v_{CGR}(\mathcal{M}).
\end{align}
If there exists a least-favorable parametric submodel, a semiparametric efficient estimator has the same asymptotic variance as the last-favorable parametric submodel.
\end{definition}

\subsection{Proving that an estimator is semiparametric efficient}

In this section, we present a sufficient condition for an estimator to be locally semiparametric efficient. Checking this condition is a standard approach to proving that an estimator is locally semiparametric efficient.

The condition requires a certain characteristic of the estimator, the \textit{influence function} (IF), and a certain characteristic of the estimand and the model, the \textit{efficient influence function} (EIF) to be defined and equal. The IF of an estimator is defined if the estimator satisfies the \textit{asymptotic linearity} property. The EIF at $P$ w.r.t. $\mathcal{M}$ of the estimand $\Psi(P)$ is defined if the estimand is \textit{pathwise differentiable} at $P$ w.r.t. $\mathcal{M}$.

\begin{definition}[Asymptotic linearity and IF] An estimator $\hat{\Psi}_n:\mathcal{X}^n \rightarrow \mathbb{R}$, based on i.i.d. sample $X_1,...,X_n$, of a parameter $\Psi(P)$ is asympototically linear at $P$, with influence function $D(P): \mathcal{X} \rightarrow \mathbb{R}$ if 
\begin{align}
\hat{\Psi}_n - \Psi(P) = \frac{1}{n}\sum_{i=1}^n D(P)(X_i) + o_P(n^{-1/2}).
\end{align}
\end{definition}

\begin{definition}[Pathwise differentiability, gradient, and EIF] The target parameter mapping/the estimand $\Psi$ is pathwise differentiable at $P$, w.r.t. $\mathcal{M}$, if there exists a function $D^0(P) \in L_2^0(P)$, (where $L_2^0(P) = \{f \in L_2(P) : Pf = 0 \}$), such that, for all parametric submodel $\{P_\epsilon : \epsilon \} \subseteq \mathcal{M}$, with score function $S$ at $\epsilon=0$ such that $P_{\epsilon=0} = P$, we have that
\begin{align}
\frac{d \Psi(P_\epsilon)}{d \epsilon}\big|_{\epsilon = 0} = P \{ D^0(P) S \}.
\end{align}
If it exists, $D^0(P)$ is called a gradient of $\Psi$ at $P$ w.r.t. $\mathcal{M}$.
The efficient influence function of $\Psi$ at $P$ w.r.t. $\mathcal{M}$, also called canonical gradient is the unique gradient of $\Psi$ at $P$ w.r.t. $\mathcal{M}$ that belongs to $T_\mathcal{M}(P)$. 
\end{definition}

\begin{proposition}\label{proposition-EIF_projection_gradient}
The EIF is the projection on $T_{\mathcal{M}}(P)$, for the $L^2(P)$ norm, of any gradient.
\end{proposition}

\begin{proposition}\label{proposition_IF_is_gradient} Consider $P \in \mathcal{M}$
Suppose $\hat{\Psi}_n$ is a regular estimator of $\Psi(P)$ w.r.t. $\mathcal{M}$, and that it is asymptotically linear with influence function $D(P)$. Then $\Psi$ is pathwise differentiable at $P$ w.r.t. $\mathcal{M}$ and $D(P)$ is a gradient of $\Psi$ at $P$ w.r.t. $\mathcal{M}$.
\end{proposition}

\begin{theorem}
If a RAL estimator has IF the EIF of the target parameter at $P$, then it is locally semiparametric efficient at $P$ for the target parameter.
\end{theorem}

These results suggest the following strategy to find the efficient influence function of a target parameter: find a RAL estimator of the target, observe that its IF is a gradient, obtain the EIF by projecting the gradient onto the tangent space.

\subsection{Explicit derivation of the EIF}

\begin{proof}
We proceed in three steps.
\paragraph{Step 1: Finding a gradient.} Denote $\Psi(P) \equiv V_1^{\pi_e}(s_1)$.
Consider 
\begin{align}
\hat{\Psi}^0_n \equiv \frac{1}{n} \sum_{i=1}^n \sum_{t=1}^T \gamma^t \rho_{1:t}^{(i)}  R_t^{(i)}.
\end{align}
Observe that 
\begin{align}
\hat{\Psi}^0_n - \Psi(P) = \frac{1}{n} \sum_{i=1}^T D^0(P)(H_i),
\end{align}
where $D^0(P)(h) = \sum_{t=1}^T \gamma^t \rho_{1:t}(h) r_t - \Psi(P)$. 

Therefore $D^0(P)(h)$ is the influence function of the estimator $\hat{\Psi}^0_n$ at $P$. It is straightforward to check that $\hat{\Psi}^0_n$ is regular. Therefore, from proposition \ref{proposition_IF_is_gradient}, $\Psi$ is pathwise differentiable at $P$ w.r.t. our statistical model $\mathcal{M}(\pi)$ and $D^0(P)$ is a gradient of $\Psi$ at $P$ w.r.t. $\mathcal{M}(\pi)$

\paragraph{Step 2: Identifying the tangent space.} Since we assumed that distributions in $\mathcal{M}(\pi)$ are dominated by a measure $\mu$, every element $P \in \mathcal{M}(\pi)$ can be represented by its density w.r.t. $\mu$, which we will denote $p$: for every $h \in \mathcal{H}$, denoting $\bar{h}_t \equiv (s_1, a_1, r_1,...,s_t, a_t, r_t)$ the history of the trajectory up till time $t$, we have 
\begin{align}
p(h) =& \frac{dP}{d\mu}(h)\\
 =& \prod_{t=1}^T \tilde{q}_t(s_{t+1}, r_t|s_t, a_t, \bar{h}_{t-1}) \prod_{t=1}^T \pi_b(a_t|s_t, ).
\end{align}
Consider a one-dimensional submodel $\{P_\epsilon : \epsilon \} \subseteq \mathcal{M}$ that passes through $P$ in $\epsilon=0$. Then
\begin{align}
\frac{d \log p_\epsilon(h)}{d \epsilon}\bigg|_{\epsilon=0} \\= \sum_{t=1}^T \frac{d \log \tilde{q}_t(s_{t+1}, r_t|s_t, a_t, \bar{h}_{t-1})}{d\epsilon} \bigg|_{\epsilon=0}.
\end{align}
Since, for any $t$, $h \mapsto \frac{d \log \tilde{q}_t(s_{t+1}, r_t|s_t, a_t, \bar{h}_{t-1})}{d\epsilon} \bigg|_{\epsilon=0}$ is a score function, it is in $L_2^0(P_{R_t, S_{t+1}|S_t, A_t, \bar{H}_{t-1} })$. Therefore,
\begin{align}
T_{\mathcal{M}(\pi)}(P) \subseteq \sum_{t=1}^T L_2^0(P_{R_t, S_{t+1}|S_t, A_t, \bar{H}_{t-1}}).
\end{align}
Conversely, for any $(g_1,...,g_T) \in L_2^0(P_{R_1, S_2|S_1, A_1}) \times ... \times L_2^0(P_{R_T, S_{T+1}|S_T, A_T, \bar{H}_{T-1} })$, for $\eta > 0$ small enough
\begin{align}
\{P_\epsilon: dP/d\epsilon = p_{\epsilon}, \epsilon \in [\eta, -\eta] \},
\end{align}
where $p_\epsilon$ is defined, for all $h \in \mathcal{H}$ as
\begin{align}
p_\epsilon(h)  = \prod_{t=1}^T & \tilde{q}_t (s_{t+1}, r_t|s_t, a_t)\\
& \times(1 + g_t(s_{t+1}, r_t, s_t, a_t))\pi_b(a_t|s_t),
\end{align}
is a submodel of $\mathcal{M}$ that passes through $P$ at $\epsilon=0$. We have that, for all $h\in \mathcal{H}$,
\begin{align}
\frac{d \log p_\epsilon}{d\epsilon} \bigg|_{\epsilon=0} = \sum_{t=1}^T g_t(s_{t+1}, r_t, a_t, s_t, \bar{h}_{t-1} ).
\end{align}
Since $\frac{d \log p_\epsilon}{d\epsilon}$ is in $T_{\mathcal{M}(\pi)}$ by definition of $T_{\mathcal{M}(\pi)}$, and that $\sum_{t=1}^T g_t$ is an arbitrary element of $\sum_{t=1}^T L_2^0(P_{R_t, S_{t+1}|S_t, A_t, \bar{H}_{t-1}})$, this shows that $\sum_{t=1}^T L_2^0(P_{R_t, S_{t+1}|S_t, A_t}) \subseteq T_{\mathcal{M}}(P)$ and therefore that
\begin{align}
T_{\mathcal{M}(P)} =\sum_{t=1}^T L_2^0(P_{R_t, S_{t+1}|S_t, A_t, \bar{H}_{t-1} })\end{align}
It is straightforward to check that the sum is direct and orthogonal.

\paragraph{Step 3: Projecting $D^0(P)$ on the tangent space.}

From proposition \ref{proposition-EIF_projection_gradient}, the EIF is given by
\begin{align}
&\Pi(D^0(P) \big| T_{\mathcal{M}(\pi)}(P))\\
 &= \sum_{t=1}^T \Pi(D^0(P) \big| L_2^0(P_{R_t, S_{t+1}|S_t, A_t, \bar{H}_{t-1} }) )\\
&= \sum_{t=1}^T \big( E_P[D^0(P)(H)|S_{t+1}, R_t, S_t, A_t, \bar{H}_{t-1} ] \\
&\qquad - E_P[D^0(P)(H)|S_t, A_t, \bar{H}_{t-1}] \big).
\end{align}
Observing that the terms that are deterministic conditional on $\bar{H}_{t-1}$ cancel out, we have that
\begin{align}
&E_P[D^0(P)(H)|S_{t+1}, R_t, S_t, A_t, \bar{H}_{t-1} ] \\
&- E_P[D^0(P)(H)|S_t, A_t, \bar{H}_{t-1}]  \\
=& E_P \left[ \sum_{\tau = t}^T \gamma^\tau \rho_{1:\tau} R_\tau | S_{t+1}, R_t, S_t, A_t, \bar{H}_{t-1} \right]\\
& - E_P\left[ \sum_{\tau = t}^T \gamma^\tau \rho_{1:\tau} R_\tau  |S_t, A_t, \bar{H}_{t-1}\right]\\
=& \gamma^t \rho_{1:t} \bigg( R_t \\
&+ \gamma E_P \left[ \sum_{\tau=t}^T \gamma^{\tau-t} \rho_{t:\tau} R_\tau \big| S_{t+1}, R_t, S_t, A_t, \bar{H}_{t-1} \right] \bigg) \\
&- \gamma^t \rho_{1:t}  E_P \left[ \sum_{\tau=t}^T \gamma^{\tau-t} \rho_{t:\tau} R_\tau \big| S_t, A_t, \bar{H}_{t-1} \right]\\\
=& \gamma^t \rho_{1:t} \bigg(R_t + \gamma V^{\pi_e}(S_{t+1}) - Q^{\pi_e}(S_t, A_t) \bigg) 
\end{align}

\paragraph{Conclusion.} The right-hand side of the last line above is equal to $D(P)$ from section \ref{section_theory_simplified_algorithm}. Since, as we see in the next section, our full-blown estimator has asymptotic variance equal to the variance of $D(P)$ which we just shown to be the EIF, it is semiparemetric efficient.

\end{proof}

\section{Cross-validated LTMLE}\label{section_CV_LTMLE}

We now present the full-blown version of our algorithm. The key difference between simplified version and the full-blown version is that the latter uses the entire dataset to fit the second-stage models, as opposed to just a split of the dataset. (Using just a split of the dataset is what the simplified algorithm does, along with other algorithms presented recently in the OPE literature.) The standard error in the simplified version scales as $1 / \sqrt{n'}$, where $n'$ is the size of the sample split used to fit $\epsilon$. With the full-blown algorithm, it scales as $1 / \sqrt{n}$, where $n$ is the size of the entire sample.

\subsection{Algorithm description}

Consider a sample $H_1,...,H_n$ of $n$ i.i.d. trajectories of the MDP. Observe that there is a one-to-one relationship between the sample $H_1,...,H_n$ and the empirical probability distribution $P_n = n^{-1} \sum_{i=1}^n \delta_{H_i}$. Therefore, we will refer to the sample and to $P_n$ interchangeably. Let $b_{1,n},...,\linebreak b_{V,n}$ be $V$ vectors in $\{0,1\}^n$ representing splits of the sample: under a given $b_{v, n}$, the training set is given by $\{i: b_{v, n}(i) = 0\}$ and the test set is given by $\{i: b_{v,n}(i) = 1 \}.$  Let $B_n$ a random vector uniformly distributed on the set $\{b_{v,n}: v=1,...,V \}$. Denote $P^0_{n, B_n}$ and $P^1_{n,B_n}$ the empirical distributions of the training set and the test set, respectively, under sample split $B_n$. Suppose that, for every $t$, we are given an estimator of $Q_t^{\pi_e}$, that is a mapping of any sample $H'_1,...,H'_n$, or equivalently, of any probability distribution $P'_n$, to a model fit, which we will denote $\hat{Q}_t^{\pi_e}(P'_n)$. In practice, $\hat{Q}_t^{\pi_e}(P'_n)$ can be the estimator of $Q_t^{\pi_e}$ obtained under a model of the dynamics fitted from trajectories $H'_1,...,H'_n$.
Denote $\sigma(x) = 1 / (1+e^{-x})$ the logistic function, and $\sigma^{-1}$ its inverse. Observe that under assumption 1, the range of $\bar{R}_{t:T}$, and therefore of $Q_t^{\pi_e}$ and $V_t^{\pi_e}$ is $[-\Delta_t, \Delta_t]$ with $\Delta_t := \sum_{\tau=t}^T \gamma^{\tau-t}$. 
For all $t$, $P'_n$, define the scaled action-value function estimator as $\tilde{Q}^{\pi_e}_t(P'_n) = (\hat{Q}^{\pi_e}(P'_n) + \Delta_t)/(2 \Delta_t).$ For $\delta \in (0,1/2)$ and any $\tilde{Q}_t':\mathcal{S} \times \mathcal{A} \rightarrow \mathbb{R}$, define
\begin{align}
    \tilde{Q}'^{\delta}_t(s_t, a_t):= \max(\delta, \min(1 - \delta, \tilde{Q}'(s_t,a_t)),
\end{align}
the thresholded version of $\tilde{Q}'$ is always at least $\delta$ away from 0 and 1. Let $\delta_n \downarrow 0.$
For any $t$, $P'_n$, introduce the perturbed scaled estimator
\begin{align}
    \tilde{Q}_t^{\pi_e}(P'_n)(\epsilon) =  \sigma (\sigma^{-1}(\tilde{Q}^{\pi_e, \delta_n}_t(P'_n)) + \epsilon )
\end{align}
and let $\hat{Q}^{\pi_e}(P'_n)(\epsilon)$ be defined by $\tilde{Q}^{\pi_e}_t(P'_n)(\epsilon) = (\hat{Q}^{\pi_e}(P'_n)(\epsilon) + \Delta_t)/(2 \Delta_t)$. 
The expression above defines a logistic regression model with a fixed offset and parameterized by an intercept $\epsilon$, such that $\hat{Q}_t^{\pi_e}(P'_n)(0) = \hat{Q}_t^{\pi_e}(P'_n)$. The TML estimate is obtained by sequentially fitting such logistic models. Specifically, start at time point $T$ and define the cross-validated risk 
\begin{align}
    &\mathcal{R}_{n,T}(\epsilon) = \\
    &E_{B_n}E_{P^1_{n,B_n}} \bigg[\rho_{1:T} \bigg(R_T \log \big(\tilde{Q}_T^{\pi_e}(P^0_{n,B_n})(\epsilon)(S_T,A_T)\big)\\
    &+ (1-R_T) \log\big(1 - \tilde{Q}_T^{\pi_e}(P^0_{n,B_n})(\epsilon)(S_T,A_T)\big) \bigg) \bigg].
\end{align}
Denote $\epsilon_{n,T}$ the minimizer of $\mathcal{R}_{n,T}(\epsilon)$. The other models are fitted by backward recursion. Suppose that we have fitted $\epsilon_{n,T},...,\epsilon_{n,t+1}$. Define $\hat{V}_{t+1}^{\pi_e}(P'_n)(\epsilon)(s) := \sum_{a} \pi_e(a|s) \hat{Q}_{t+1}^{\pi_e}(P'_n)(\epsilon)(s, a)$, $U_{t,B_n} = R_t + \gamma \hat{V}_{t+1}^{\pi_e}(P^0_{n,B_n})(\epsilon_{n, t+1})(S_{t+1})$ and $\tilde{U}_{t,B_n} = (U_{t,B_n} + \Delta_t) / (2 \Delta_t)$. Define the cross-validated risk
\begin{align}
    &\mathcal{R}_{n,t}(\epsilon) = \\
    &E_{B_n}E_{P^1_{n,B_n}}\bigg[\rho_{1:t} \bigg(\tilde{U}_{t,B_n} \log \big(\tilde{Q}_t^{\pi_e}(P^0_{n,B_n})(\epsilon)(S_t,A_t)\big) \\
    &+ (1-\tilde{U}_{t,B_n}) \log\big(1 - \tilde{Q}_t^{\pi_e}(P^0_{n,B_n})(\epsilon)(S_t,A_t)\big) \bigg) \bigg].
\end{align}
The perturbation $\epsilon_{n,t}$ is defined as the minimizer of the above risk. The TML estimator of $V^{\pi_e}_1(s_1)$ is defined as 
\begin{align}
    \hat{V}_1^{\pi_e, TMLE}(s_1) := E_{B_n} \hat{V}_1^{\pi_e}(P^0_{n,B_n})(\epsilon_1)(s_1).
\end{align}

\begin{theorem}
Suppose assumptions \ref{assumption_absolute_continuity}, \ref{assumption_bounded_rewards}, \ref{assumption_convergence_initial_estimator}, \ref{assumption_bounds_limit_initial_estimator}, \ref{assumption_bound_IS_ratios} are satisfied. Then
\begin{align}
    E_{P^{\pi_b}}[\hat{V}^{\pi_e,TMLE}(s_1) - V_1^{\pi_e}(s_1)] = o(1 / \sqrt{n}),
\end{align}
and
\begin{align}
    &\sqrt{n} \left(\hat{V}^{\pi_e,TMLE}(s_1) - V_1^{\pi_e}(s_1) = O_P(1 / \sqrt{n}) \right) \\
    &\xrightarrow{d} \mathcal{N}(0, \sigma^2(Q_\infty(\epsilon_\infty, \rho)),
\end{align}
where, for all non-random $Q'$ and $\rho'$
\begin{align}
    \sigma^2(Q', \rho') = Var_{P^{\pi_b}}(D(Q', \rho')).
\end{align}
\end{theorem}

It has been established in earlier works \cite{Jiang2015} that the DR estimator with initial estimator $\hat{Q}$ also have asymptotic variance $\sigma^2(Q_\infty)/ n$, and that $\sigma^2(Q_\infty)/ n$ is the efficient variance from the Cramer-Rao lower bound, provided $Q_\infty =Q$ (that is provided the initial estimator $\hat{Q}$ is asymptotically consistent.). If the initial estimator $\hat{Q}$ is consistent $\epsilon_\infty = 0$ and $Q_\infty(\epsilon_\infty)=Q_\infty=Q$, therefore the DR estimator and the LTMLE have the same asymptotic distribution and they achieve the Cramer-Rao lower bound.

\subsection{Additional notation}

For a given policy $\pi$ and a transition probability $P$, denote $P^\pi$ the corresponding probability distribution over a trajectory with fixed inital state, that is, for all $h=(s_1,a_1,r_1,...,s_T, a_T, r_T)$,
\begin{align}
    P^{\pi_b}(H=h) = \prod_{t=1}^T P(r_t,s_{t+1}|s_t, a_t) \pi(a_t|s_t).
\end{align}

From now, we will denote $Q':=(Q'_1,...,Q'_T)$ an arbitrary action-value function, and let $V'=(V'_1,...,V'_T)$ be the corresponding value function under $\pi_e$, that is for all $t, s_t$, $V'_t(s_t) = \sum_{a_t'} \pi_e(a'_t|s_t) Q'(s_t, a_t')$. We will also drop the $\pi_e$ subscript whenever possible and denote $Q_t := Q^{\pi_e}_t$, the true action value function at time $t$, and similarly, we will denote $V_t:=V^{\pi_e}_t$, the true value function at time $t$. Denote $Q=(Q_1,...,Q_T)$ and $V=(V_1,...,V_T)$.

We introduce the following notation for the perturbed estimators: denote $\hat{Q}^*_t(P^0_{n, B_n}) := \hat{Q}^{\pi_e}_t(P^0_{n,B_n})(\epsilon_{n,t})$, $\tilde{Q}^*_t(P^0_{n,B_n}) := \tilde{Q}^{\pi_e}_t(P^0_{n,B_n})(\epsilon_{n,t})$, $\hat{V}^*_t(P^0_{n,B_n})(\cdot) = \sum_{a'_t} \pi_e(a'_t|\cdot)  \tilde{Q}^*_t(P^0_{n,B_n})(a'_t|\cdot).$ Finally, define $\hat{Q}^*(P^0_{n,B_n}) := (\hat{Q}^*_1(P^0_{n,B_n}),...,\hat{Q}^*_T(P^0_{n,B_n}))$, $\hat{V}^*(P^0_{n,B_n}) := (\hat{V}^*_1(P^0_{n,B_n}),...,\hat{V}^*_{T+1}(P^0_{n,B_n}))$.

\subsection{The fits of the second-stage models solve a score equation}

\begin{lemma}\label{lemma_targeting_solves_EICeq}
The perturbed estimators $\hat{Q}^*$, $\hat{V}^*$ given by the LTMLE algorithm satisfy
\begin{align}
    E_{B_n} P^1_{n,B_n} D^*(\hat{Q}^*(P^0_{n,B_n}), \hat{V}^*(P^0_{n,B_n})) = 0.
\end{align}
\end{lemma}

\begin{proof}
Defining $\tilde{U}_{T, B_n} := R_T$, we have that, for all $t=1,...,T$,
\begin{align}
    l_{t,n}(\epsilon) = -E_{B_n} P^1_{n,B_n} f_{B_n}(\epsilon),
\end{align}
with 
\begin{align}
    &f_{B_n}(\epsilon)(H) := \\
    \rho_{1:t}(&-\tilde{U}_{t,B_n} \log \sigma(a_{B_n} + \epsilon) \\
    &- (1-\tilde{U}_{t,B_n}) \log(1 - \sigma  (a_{B_n} + \epsilon))),
\end{align}
where $a_{B_n} := \sigma^{-1}(\tilde{Q}^{\pi_e}(P^0_{n,B_n})(S_t,A_t))$.
Using the expression of $\sigma$, we rewrite $f_{B_n}(\epsilon)$ as 
\begin{align}
    &f_{B_n}(\epsilon)(H) = \\
    \rho_{1:t}(&\tilde{U}_{t, B_n} \log (1 + e^{-a_{B_n} - \epsilon}) \\
    &+ (1 - \tilde{U}_{t, B_n}) \log(1 + e^{a_{B_n} + \epsilon})).
\end{align}
We take the derivative of $f_{B_n}$ w.r.t. $\epsilon$:
\begin{align}
    &f'_{B_n}(\epsilon)(H) \\
    =& \rho_{1:t} \left(- \tilde{U}_{t, B_n} \frac{e^{-a_{B_n} - \epsilon}}{1 + e^{-a_{B_n} - \epsilon}} + (1 - \tilde{U}_{t,B_n}) \frac{e^{a_{B_n} + \epsilon}}{1+e^{a_{B_n} + \epsilon}}\right)\\
    =& \rho_{1:t}\left(- \tilde{U}_{t, B_n} (1 - \sigma(a_{B_n} + \epsilon) )  + (1-\tilde{U}_{t,B_n}) \sigma(a_{B_n} + \epsilon)\right)\\
    =& \rho_{1:t} \left(\sigma(a_{B_n} + \epsilon) - \tilde{U}_{t,B_n}\right).
\end{align}
Recalling the definitions of $a_{B_n}$, $\tilde{U}_{t, B_n}$, and $\tilde{Q}^{\pi_e}$, we rewrite the above expression as
\begin{align}
 &f'_{B_n}(\epsilon)(H) \\
 =& \rho_{1:t}\left(\tilde{Q}^{\pi_e}(P^0_{n,B_n})(\epsilon)(S_t,A_t) - \tilde{U}_{t,B_n}\right)\\
    =& (2 \Delta_t)^{-1} \rho_{1:t} \bigg( \hat{Q}_t^{\pi_e}(P^0_{n,B_n})(\epsilon)(S_t,A_t)\\
    &\qquad \qquad \qquad - R_t - \gamma \hat{V}_{t+1}^{\pi_e}(P^0_{n,B_n})(\epsilon_{t+1})(S_{t+1}) \bigg)
\end{align}
Since $\epsilon_t$ verifies $l'_{n,t}(\epsilon)=0$, we have that
\begin{align}
E_{B_n} E_{P^1_{n,B_n}}\bigg[ \rho_{1:t} \bigg(&R_t + \gamma \hat{V}_{t+1}^{\pi_e}(P^0_{n,B_n})(\epsilon_{t+1})(S_{t+1})\\
&- \hat{Q}^{\pi_e}_t(P^0_{n,B_n})(\epsilon_t)(S_t,A_t)\bigg)\bigg]=0,
\end{align}
that is 
\begin{align}
    E_{B_n}P^1_{n,B_n} D_t(\hat{Q}^*_t(P^0_{n,B_n}), \hat{V}^*_t(P^0_{n,B_n})) = 0.
\end{align}
Summing over $t$ yields the result.
\end{proof}

\subsection{Proof of convergence of the perturbations}

\begin{lemma}\label{lemma_bound_dphi_dx}
Define, for all $x\in(0,1)$, $\epsilon \in \mathbb{R}$,
\begin{align}
    \phi_1(\epsilon, x) :=& \log (\sigma(\sigma^{-1}(x) + \epsilon))\\
      \text{and } \phi_2(\epsilon, x) :=& \log (1 - \sigma(\sigma^{-1}(x) + \epsilon)).
\end{align}
It holds that, for all $x \in (0,1)$, $\epsilon \in \mathbb{R}$,
\begin{align}
    \frac{\partial \phi_1}{\partial x}(\epsilon, x) &= \left(\frac{1}{x} + \frac{1}{1-x}\right)(1 - \sigma(\sigma^{-1}(x) + \epsilon)),\\
    \text{and } \frac{\partial \phi_2}{\partial x}(\epsilon, x) &= \left(\frac{1}{x} + \frac{1}{1-x}\right)\sigma(\sigma^{-1}(x) + \epsilon).
\end{align}
Therefore, if $x \in [\delta, 1 - \delta]$ for some $\delta \in (0,1/2)$ we have that for all $\epsilon \in \mathbb{R}$,
\begin{align}
    \bigg|\frac{\partial \phi_1}{\partial x}(\epsilon, x) \bigg| \leq 2 \delta^{-1} \qquad \text{ and } \qquad \bigg|\frac{\partial \phi_2}{\partial x}(\epsilon, x) \bigg| \leq 2 \delta^{-1}.
\end{align}
\end{lemma}

\begin{lemma}\label{lemma_bound_phi}
Consider $\phi_1$ and $\phi_2$ as in lemma \ref{lemma_bound_dphi_dx} above, and suppose that $x \in [\delta, 1-\delta]$, for some $\delta \in (0,1/2)$. It holds that for all $\epsilon \in \mathbb{R}$
\begin{align}
    |\phi_1(\epsilon, x)| &\leq \log (1 + \delta^{-1} e^\epsilon),\\
    \text{and } |\phi_2(\epsilon, x)| &\leq \log (1 + \delta^{-1} e^\epsilon).
\end{align}
\end{lemma}

\begin{lemma}\label{lemma_pointwise_targeting_risk_convergence}
Assume that for all $h \in \mathcal{H}$, for all $t=1,..,T$, $0 \leq \rho_{1:t}(h) \leq M$ for some $M > 0$. Assume that for all $\|\hat{Q}_t(P'_n) - Q_{\infty,t}\|_{P^{\pi_b},2} = o_P(\delta_n)$ for some $\delta_n \downarrow 0$. Assume that for all $t=1,...,T$, for all $_t, a_t \in \mathcal{S} \times \mathcal{A}$, $\tilde{Q}_{t,\infty}(s_t,a_t) \in [\delta, 1-\delta]$ for some $\delta \in (0,1/2)$. Then, for all $\epsilon \in \mathbb{R}$,
\begin{align}
\mathcal{R}_{n,t}(\epsilon) - \mathcal{R}_{\infty,t}(\epsilon) = o_P(1).
\end{align}
\end{lemma}

\begin{proof} Let $\epsilon \in \mathbb{R}$.
We express the risk $\mathcal{R}_{n,t}$ as a cross-validated empirical mean of a loss, and the risk $\mathcal{R}_\infty$ as the population mean of a loss:
\begin{align}
    \mathcal{R}_{n,t}(\epsilon) =& E_{B_n}P^1_{n,B_n} l_t(\tilde{Q}^{\delta_n}_t(P^0_{n,B_n})(\epsilon)),\\
    \text{and }\mathcal{R}_{\infty, t}(\epsilon) =& P^{\pi_b} l_t(\tilde{Q}_{\infty, t}(\epsilon)),
\end{align}
where, for all $\tilde{Q}'_t:\mathcal{S}\times\mathcal{A} \rightarrow (0,1)$, for all $h \in \mathcal{H}$
\begin{align}
    &l_t(\tilde{Q}'_t)(h) :=\\
     \rho_{1:t}(h) \bigg(& \tilde{u}_{t,n,B_n} \log(\tilde{Q}'_t(s_t,a_t))\\
    &+ (1 - \tilde{u}_{t,n,B_n}) \log(1 - \tilde{Q}'_t(s_t,a_t)) \bigg).
\end{align}
From there, we are going to proceed in three steps: in the first steps, we will first decompose $\mathcal{R}_{n,t}(\epsilon) - \mathcal{R}_{\infty, t}(\epsilon)$ in two terms $A_{n,t}$ and $B_{n,t}$, that we will then each bound separately in the second and third step.

\paragraph{Step 1: decomposition of $\mathcal{R}_{n,t}(\epsilon) - \mathcal{R}_{\infty, t}(\epsilon)$.} Observe that
\begin{align}
    \mathcal{R}_{n,t}(\epsilon) - \mathcal{R}_{\infty, t}(\epsilon) =& A_{n,t} + B_{n,t},
\end{align}
with 
\begin{align}
    A_{n,t} =& E_{B_n} (P^1_{n,B_n} -P^{\pi_b}) l_t(\tilde{Q}^{\delta_n}_t(P^0_{n,B_n})(\epsilon)),
\end{align}
and
\begin{align}
 B_{n,t} =& E_{B_n} P^{\pi_b} \left(l_t(\tilde{Q}^{\delta_n}_t(P^0_{n,B_n})(\epsilon)) -  l_t(\tilde{Q}_{\infty, t}(\epsilon))\right).
\end{align}

\paragraph{Step 2: bounding $A_{n,t}$.} Let $n_0 = \left\lfloor np \right\rfloor$ and $n_1 = n - n_0$. Denote $H^0_{B_n, 1},...,H^0_{B_n, n_0}$, the trajectories in the training set and $H^1_{B_n, 1},...,H^1_{B_n, n_1}$ the trajectories in the test set corresponding to sample split $B_n$. 

 Since $\tilde{Q}_t^{\delta_n} \in [\delta_n, 1 - \delta_n]$, lemma \ref{lemma_bound_phi} shows that 
\begin{align}
    |\log(\tilde{Q}_t^{\delta_n}(\epsilon)(H^1_{B_n,i}))| \leq & \log (1 +\delta_n^{-1} e^{\epsilon}) \\
    \lesssim & \log(1 /\delta_n),
\end{align}
and
\begin{align}
|\log(1-\tilde{Q}_t^{\delta_n}(\epsilon)(H^1_{B_n,i}))| \leq& \log (1 +\delta_n^{-1} e^{\epsilon}) \\
\lesssim  & \log(1 /\delta_n).
\end{align}
Recalling the expression of $l_{n,t}$, the fact that by assumption, for every $i=1,...,n_1$, $\rho_{1:t}(H^1_{B_n,i}) \leq M$ almost surely, and the fact that $\tilde{U}_{t,n,B_n}(H^1_{B_n,i}) \in [0,1]$, we can bound the loss as follows:
\begin{align}
    |l_t(\tilde{Q}^{\delta_n}_t(P^0_{n,B_n})(\epsilon))(H^1_{B_n,i})| \lesssim M \log(1 / \delta_n),
\end{align}
almost surely, for every $i=1,...,n_1$.
Conditional on $H^0_{B_n, 1},...,H^0_{B_n, n_0}$, $\
\linebreak l_t(\tilde{Q}^{\delta_n}_t(P^0_{n,B_n})(\epsilon))(H^1_{B_n,1}),...,l_t(\tilde{Q}^{\delta_n}_t(P^0_{n,B_n})(\epsilon))(H^1_{B_n,n_1})$ are i.i.d. random variables upper bounded, up to a constant, by $M \log(1 /\delta_n)$. Therefore, by Hoeffding's inequality, for every $x > 0$,
\begin{align}
    &P[|(P^1_{n,B_n} - P^{\pi_b}) l_t(\tilde{Q}^{\delta_n}_t(P^0_{n,B_n})(\epsilon))| > x]\\
     \leq & 2 \exp\left(-\frac{nx^2}{2 \log(1/\delta_n)}\right).
\end{align}
Therefore,
\begin{align}
    &(P^1_{n,B_n} - P^{\pi_b}) l_t(\tilde{Q}^{\delta_n}_t(P^0_{n,B_n})(\epsilon))\\
     = & O_P\left(\sqrt{\frac{\log(1/\delta_n)}{n}} \right).
\end{align}
Since $B_n$ takes finitely many values, and that $\log(1 / \delta_n) = o(n)$, the above display implies that
\begin{align}
     E_{B_n}(P^1_{n,B_n} - P^{\pi_b}) l_t(\tilde{Q}^{\delta_n}_t(P^0_{n,B_n})(\epsilon)) = o_P(1).
\end{align}

\paragraph{Step 3: bounding $B_{n,t}$.}
Since $\rho_{1:t} \leq M$ for every $t$ almost surely under $P^{\pi_b}$, there exists a subset $\bar{\mathcal{H}}$ of $\mathcal{H}$ such that $H \in \bar{\mathcal{H}}$ almost surely, and for all $h\in \bar{\mathcal{H}}$, $\rho_{1:t}(h) \leq M$ for every $t$. As far as integrals w.r.t. are concerned, $P^{\pi_b}$, it is enough to characterize the integrands on $\bar{\mathcal{H}}$.
Let $h$ be an arbitrary element of $\bar{\mathcal{H}}$. 
\begin{align}
    &l_t(\tilde{Q}_t^{\delta_n}(P^0_{n,B_n})(\epsilon)) - l_t(\tilde{Q}_{\infty,t}(\epsilon)) \\
    &= \rho_{1:t}(h) \bigg\{ \tilde{u}_{t,n,B_n}(h) \bigg( \log (\tilde{Q}_t^{\delta_n}(P^0_{n,B_n})(\epsilon)(h))\\
    &\qquad\qquad - \log(\tilde{Q}_{\infty,t}(\epsilon)(h)) \bigg)\\
    &+ (1-\tilde{u}_{t,n,B_n}(h)) \bigg( \log (1-\tilde{Q}_t^{\delta_n}(P^0_{n,B_n})(\epsilon)(h)) \\
    &\qquad\qquad- \log(1-\tilde{Q}_{\infty,t}(\epsilon)(h))\bigg) \bigg\} . \label{proof_lemma_convergence_risk_eq1}
\end{align}
From lemma \ref{lemma_bound_dphi_dx} and the mean value theorem, for all $n$ such that $\delta_n \leq \delta$,
\begin{align}
&\big|\log (\tilde{Q}_t^{\delta_n}(P^0_{n,B_n})(\epsilon)(h)) - \log(\tilde{Q}_{\infty,t}(\epsilon)(h))\big| \\
&\leq 2 \delta_n^{-1} \big|\tilde{Q}_t^{\delta_n}(P^0_{n,B_n})(h) - \tilde{Q}_{\infty,t}(h))\big|\\
&\leq 2 \delta_n^{-1} \big|\tilde{Q}_t(P^0_{n,B_n})(h) - \tilde{Q}_{\infty,t}(h))\big|\\
&\leq 2 \delta_n^{-1} (2\Delta_t)^{-1} \big|\hat{Q}_t(P^0_{n,B_n})(h) - Q_{\infty,t}(h))\big|. \label{proof_lemma_convergence_risk_ineq1}
\end{align}
The third line above follows from the fact that, for all $x \in [0,1]$, $y \in [\delta, 1-\delta]$, and $n$ such that $\delta_n \leq \delta$, it holds that $|\max(\delta_n, \min(1-\delta_n, x)) - y| \leq |x-y|$.
The same reasoning shows that
\begin{align}
    &\big|\log (1 - \tilde{Q}_t^{\delta_n}(P^0_{n,B_n})(\epsilon)(h)) - \log(1 - \tilde{Q}_{\infty,t}(\epsilon)(h))\big|\\
    &\leq 2 \delta_n^{-1} (2\Delta_t)^{-1} \big|\hat{Q}_t(P^0_{n,B_n})(\epsilon)(h) - Q_{\infty,t}(\epsilon)(h))\big| \label{proof_lemma_convergence_risk_ineq2}.
\end{align}\
Taking the absolute value of \eqref{proof_lemma_convergence_risk_eq1}, using the triangle inequality, the fact that $0 \leq \rho_{1:t}(h) \leq M$, that $\tilde{u}_{t,n,B_n}(h) \in [0,1]$ and the upper bounds \eqref{proof_lemma_convergence_risk_ineq1} and \eqref{proof_lemma_convergence_risk_ineq2} yields
\begin{align}
    &\big| l_t(\tilde{Q}_t^{\delta_n}(P^0_{n,B_n})(\epsilon)) - l_t(\tilde{Q}_{\infty,t}(\epsilon)) \big| \\
    & \leq M 2 \delta_n^{-1} \Delta_t^{-1}\big|\hat{Q}_t(P^0_{n,B_n})(h) - Q_{\infty,t}(h))\big|.
\end{align}
Therefore, using the triangle inequality and Cauchy-Schwartz, and the fact that $B_n$ takes finitely many values,
\begin{align}
    &\big| E_{B_n} P^{\pi_b} ( l_t(\tilde{Q}_t^{\delta_n}(P^0_{n,B_n})(\epsilon)) - l_t(\tilde{Q}_{\infty,t}(\epsilon))) \big| \\
    &\leq M 2 \delta_n^{-1} \Delta_t^{-1} \|\hat{Q}_t(P^0_{n,B_n}) - Q_{\infty,t}\|_{P^{\pi_b}, 2}.
\end{align}
Therefore, using the assumption that $\|\hat{Q}_t(P^0_{n,B_n}) - Q_{\infty,t}\|_{P^{\pi_b}, 2} = o_P(\delta_n)$, we have
\begin{align}
    E_{B_n} P^{\pi_b} ( l_t(\tilde{Q}_t^{\delta_n}(P^0_{n,B_n})(\epsilon)) - l_t(\tilde{Q}_{\infty,t}(\epsilon))) = o_P(1).
\end{align}

Therefore, putting together that $A_{n,t}=o_P(1)$ and $B_{n,t}=o_P(1)$ and the fact that $\mathcal{R}_{n,t} - \mathcal{R}_{\infty,}(\epsilon) = A_{n,t} + B_{n.t}$ gives the desired result.
\end{proof}

\begin{lemma}\label{lemma_convergence_epsilon}
Make the same assumptions as in lemma \ref{lemma_pointwise_targeting_risk_convergence} above. Then
\begin{itemize}
    \item $\mathcal{R}_{\infty, t}$ has a unique minimizer $\epsilon_{\infty,t}$,
    \item $\epsilon_{n,t} - \epsilon_{\infty,t} = o_P(1)$.
\end{itemize}
\end{lemma}

\begin{proof}\label{proof_lemma_convergence_epsilon}
Let $\eta > 0$ and $\kappa > 0$. The fact that $Q_{\infty, t} \in [\delta, 1- \delta]$, with $\delta \in (0,1/2)$ implies that $\mathcal{R}_{\infty, t}$ is $m$-strongly convex for some $m>0$. Therefore $\mathcal{R}_{n,t}$ has a unique minimizer on $\mathbb{R}$ that we will denote $\epsilon_{\infty, t}$. Denoting $\Delta := m \eta^2  /2$, we have, from $m$-strong convexity, that
\begin{align}
    \mathcal{R}_{\infty, t}(\epsilon_{\infty, t} + \eta) &\geq \mathcal{R}_{\infty, t}(\epsilon_{\infty, t}) +  \Delta, \label{proof_lemma_convergence_epsilon_eq1}\\
    \text{ and }\mathcal{R}_{\infty, t}(\epsilon_{\infty, t} - \eta) &\geq \mathcal{R}_{\infty, t}(\epsilon_{\infty, t}) +  \Delta. \label{proof_lemma_convergence_epsilon_eq2}
\end{align}
Consider the following event:
\begin{align}
    \mathcal{E} := \bigg\{ & |\mathcal{R}_{n,t}(\epsilon) - \mathcal{R}_{\infty, t}(\epsilon)| \leq \frac{\Delta }{3},\\
    &  \forall \epsilon \in \{\epsilon_\infty, \epsilon_\infty - \eta, \epsilon_\infty + \eta \} \bigg\}.
\end{align}
From the pointwise convergence in probability of $\mathcal{R}_{n,t}$, which is given to us by lemma \ref{lemma_pointwise_targeting_risk_convergence} above, there exists $n_0$ such that for all $n \geq n_0$, $P[\mathcal{E}] \geq 1 - \kappa$. Assume $\mathcal{E}$ holds. Then, from \eqref{proof_lemma_convergence_epsilon_eq1} and \eqref{proof_lemma_convergence_epsilon_eq2}, and the inequalities that define event $\mathcal{E}$, we have that 
\begin{align}
    \mathcal{R}_{n,t}(\epsilon_\infty \pm \eta) \geq \mathcal{R}_{n,t}(\epsilon_\infty) + \frac{\Delta}{3}.
\end{align}
From convexity of $\mathcal{R}_{n,t}$, the above display implies that for all $\epsilon$ such that $|\epsilon - \epsilon_{\infty, t}| \geq \eta$, we have that 
\begin{align}
    \mathcal{R}_{n,t}(\epsilon) \geq \mathcal{R}_{n,t}(\epsilon_{\infty, t}) + \frac{\Delta}{3}.
\end{align}
Since $\epsilon_{n, t}$ minimizes $\mathcal{R}_{n,t}$, we must have $\mathcal{R}_{n,t} \leq \mathcal{R}_{n,t}(\epsilon_{\infty, t}) < \mathcal{R}_{n,t}(\epsilon_{\infty, t}) + \Delta / 3 $. Therefore, if $\mathcal{E}$ is realized, $\epsilon_{n, t}$ must lie in $[\epsilon_{\infty, t} - \eta, \epsilon_{\infty, t} + \eta]$. Since $\mathcal{E}$ is realized with probability $1 - \kappa$, this concludes the proof.
\end{proof}

\subsection{Proof of theorem 2}

The proof relies on the following empirical process result.

\begin{lemma}\label{lemma_equicontinuity}
Consider $\mathcal{F}_\eta(P^0_{n,B_n})$ as defined in the previous section. Consider $\eta_n = o_P(1)$. Then
\begin{align}
\sqrt{n} \sup_{f\in \mathcal{F}_{\eta_n}(P^0_{n,B_n})}|(P^1_{n,B_n} - P^{\pi_b})f| = o_P(1).
\end{align}
\end{lemma}

\begin{proof}
This is a direct corollary of lemma \ref{lemma_simple_algo_equicontinuity}.
\end{proof}

\begin{proof}[Proof of theorem 2.]
Recall that $\hat{V}_1^{\pi_e, TMLE}(s_1) = E_{B_n}\hat{V}(P^0_{n,B_n})(\epsilon_{n,1})(s_1)$.
Therefore, from lemma \ref{lemma_simple_algo_1st_order_expansion},
\begin{align}
    &\hat{V}_1^{\pi_e, TMLE}(s_1) - V^{\pi_e}(s_1) \\
    =& -E_{B_n} P^{\pi_b} D(\hat{Q}(P^0_{n,B_n})(\epsilon_n), \hat{V}(P^0_{n,B_n})(\epsilon_n)). \label{pf_thm1_eq1}
\end{align}
Recall that from lemma \ref{lemma_targeting_solves_EICeq},
\begin{align}
    &\hat{V}_1^{\pi_e, TMLE}(s_1) - V^{\pi_e}(s_1) \\
    =& E_{B_n} P^1_{n,B_n} D(\hat{Q}(P^0_{n,B_n})(\epsilon_n), \hat{V}(P^0_{n,B_n})(\epsilon_n)). \label{pf_thm1_eq2}
\end{align}
Summing \eqref{pf_thm1_eq1} and \eqref{pf_thm1_eq2} yields
\begin{align}
    &\hat{V}_1^{\pi_e, TMLE}(s_1) - V^{\pi_e}(s_1) =\\
    &E_{B_n} (P^1_{n,B_n} -  P^{\pi_b}) D(\hat{Q}(P^0_{n,B_n})(\epsilon_n), \hat{V}(P^0_{n,B_n})(\epsilon_n)).
\end{align}
Using the notation $f_\epsilon$ introduced in the previous section, we can rewrite the above expression as
\begin{align}
&\hat{V}_1^{\pi_e, TMLE}(s_1) - V^{\pi_e}(s_1)\\
 &= E_{B_n} (P^1_{n,B_n} -  P^{\pi_b}) f_{\epsilon_\infty}(P^0_{n,B_n}) \\
&+ E_{B_n} (P^1_{n,B_n} -  P^{\pi_b}) (f_{\epsilon_n}(P^0_{n,B_n}) - f_{\epsilon_\infty}(P^0_{n,B_n})). \label{pf_thm1_eq3}
\end{align}
By the central limit theorem for triangular arrays  
\begin{align}
    &\sqrt{n}(P^1_{n,B_n} -  P^{\pi_b}) f_{\epsilon_\infty}(P^0_{n,B_n}) \\
    &\xrightarrow{d} \mathcal{N}(0, Var( D(Q_\infty(\epsilon_\infty), V_\infty(\epsilon_\infty)(H))).
\end{align}
Therefore, $\sqrt{n}(P^1_{n,B_n} -  P^{\pi_b}) f_{\epsilon_\infty}(P^0_{n,B_n}) = O_P(n^{-1/2})$.
The second term is the RHS of \label{pf_thm1_eq3} is $o_P(n^{-1/2})$ by lemma \ref{lemma_equicontinuity}.
Since $B_n$ takes values on a finite support, this implies that
\begin{align}
\hat{V}_1^{\pi_e, TMLE}(s_1) - V^{\pi_e}(s_1) = O_P(n^{-1/2}).
\end{align}

\end{proof}

\section{Experiment Details}
In this section, we provide full details of our experiments and utilized domains. In particular, we provide detailed descriptions of discrete-state domains ModelWin, ModelFail and Gridworld. 

\subsection{ModelWin}
The ModelWin environment was constructed in order to simulate situations in which the approximate model of the MDP will converge quickly to the truth. On the other hand, importance-sampling based methods might suffer from high variance. 

The ModelWin MDP consists of 3 states, and the agent always begins at state $s_1$. At $s_1$, the agent stochastically picks between two actions, $a_1$ and $a_2$. Under action $a_1$, the agent transitions to $s_2$ with probability 0.4 and $s_3$ with probability 0.6. On the other hand, under action $a_2$ the agent does the opposite- it transitions to $s_2$ and $s_3$ with probability 0.6 and 0.4, respectively. Under both actions, if the agent transitions to $s_2$, it gets a positive reward of +1. Consequently $s_1$ to $s_3$ transitions are penalized with -1 reward. In states $s_3$ and $s_2$, both actions $a_1$ and $a_2$ will take the agent back to $s_1$ with probability 1 and no reward. The horizon is set to $T=20$.  

The considered behavior policy takes action $a_1$ from $s_1$ with probability 0.73, and action $a_2$ with probability 0.27. The evaluation policy has the opposite behavior. Note that both the behavior and evaluation policies select actions uniformly at random while in states $s_1$ and $s_2$.

\subsection{ModelFail}

Unlike the ModelWin domain, the agent does not observe the true underlying states of the MDP in ModelFail. The purpose of this domain is to test environments are not known perfectly, and where the approximate model will fail to converge to the true MDP. ModelFail attempts to mimic partial observability, common in real applications. 

The actual MDP consists of 4 states, 3 states and a final absorbing state, however the agent is not able to distinguish between them. The agent always starts at the same state, $s_1$, where it has two actions available. With actions $a_1$ it transitions into the upper state ($s_2$), whereas with action $a_2$ it goes to the lower state ($s_3$). No matter which state the agent transitioned to, both $s_2$ and $s_3$ lead to the terminal absorbing state $s_4$. However, $s_2$ to $s_4$ transition carries reward +1, whereas $s_3$ to $s_4$ leads to reward of -1. The horizon is $T=2$. 

The considered behavior policy takes action $a_1$ with probability 0.88, and action $a_2$ with probability 0.22. The evaluation policy has the opposite behavior.

\subsection{Gridworld}

The last discrete-state environment used is a $4 \times 4$ gridworld domain with 4 actions (up, down, left, right) developed by \cite{ThomasThesis}. As emphasized by \cite{thomas2016}, this is a domain specifically developed for evaluation of OPE estimators. However, due to its deterministic nature, it will favor model-based approaches.

The horizon for GridWorld is $T=100$, after which the episode ends unless the terminal state of $s_{12}$ is reached before $T$. The reward is always -1, expect at states $s_8$ where it is +1, $s_{12}$ with +10, and $s_6$ where the agent is penalized with -10 reward. 

We used two different polices for the gridworld, as described in \cite{ThomasThesis}. In particular, policy $\pi_1$ selects each of the 4 actions with equal probability regardless of the observation. Intuitively this policy takes a long time to reach the goal, and potentially often visits the state with the maximum negative reward. In addition, we also considered the near-optimal+ policy $\pi_5$, which exemplifies a near-deterministic near-optimal policy that moves quickly to $s_8$ with reward +1, without visiting $s_6$ with -10 reward. At $s_8$ it chooses action down with high probability, collecting as many positive rewards as possible until the time limit runs out. Once it eventually chooses the right action, it moves almost deterministically to $s_{12}$ where it collects its final reward and end the episode. 

\end{document}